\documentclass[numbers]{article}%[12pt]

\usepackage[final]{neurips_2023}

\usepackage{natbib}
\usepackage[backref=page]{hyperref}
\hypersetup{
    colorlinks=true,
    linkcolor=orange,
    citecolor=blue,
    urlcolor=blue,
    backref=page
}

\usepackage{url}            % simple URL typesetting
\usepackage{booktabs}       % professional-quality tables
\usepackage{amsfonts}       % blackboard math symbols
\usepackage{nicefrac}       % compact symbols for 1/2, etc.
\usepackage{microtype}      % microtypography
\usepackage{algolyx}        % algorithm

\usepackage[english]{babel}
\usepackage{graphicx}
\usepackage{wrapfig}

\usepackage{morenotations}

\floatstyle{ruled}
\newfloat{algorithm}{tbp}{loa}
\providecommand{\algorithmname}{Algorithm}
\floatname{algorithm}{\protect\algorithmname}

\title{Distributionally Robust Bayesian Optimization\\ with $\varphi$-divergences}

\author{
  Hisham Husain \\
Amazon \\
\texttt{hushisha@amazon.com}
\And
Vu Nguyen \\
Amazon \\
\texttt{vutngn@amazon.com} \\
\And
Anton van den Hengel \\
Amazon \\
\texttt{hengelah@amazon.com}
}

\begin{document}

\maketitle

\global\long\def\se{\hat{\text{se}}}%

\global\long\def\interior{\text{int}}%

\global\long\def\boundary{\text{bd}}%

\global\long\def\new{\text{*}}%

\global\long\def\stir{\text{Stirl}}%

\global\long\def\dist{d}%

\global\long\def\HX{\entro\left(X\right)}%
 
\global\long\def\entropyX{\HX}%

\global\long\def\HY{\entro\left(Y\right)}%
 
\global\long\def\entropyY{\HY}%

\global\long\def\HXY{\entro\left(X,Y\right)}%
 
\global\long\def\entropyXY{\HXY}%

\global\long\def\mutualXY{\mutual\left(X;Y\right)}%
 
\global\long\def\mutinfoXY{\mutualXY}%

\global\long\def\xnew{y}%

\global\long\def\bx{\mathbf{x}}%

\global\long\def\bz{\mathbf{z}}%

\global\long\def\bu{\mathbf{u}}%

\global\long\def\bs{\boldsymbol{s}}%

\global\long\def\bk{\mathbf{k}}%

\global\long\def\bX{\mathbf{X}}%

\global\long\def\tbx{\tilde{\bx}}%

\global\long\def\by{\mathbf{y}}%

\global\long\def\bY{\mathbf{Y}}%

\global\long\def\bZ{\boldsymbol{Z}}%

\global\long\def\bU{\boldsymbol{U}}%

\global\long\def\bv{\boldsymbol{v}}%

\global\long\def\bn{\boldsymbol{n}}%

\global\long\def\bV{\boldsymbol{V}}%

\global\long\def\bK{\boldsymbol{K}}%

\global\long\def\bw{\vt w}%

\global\long\def\bbeta{\gvt{\beta}}%

\global\long\def\bmu{\gvt{\mu}}%

\global\long\def\btheta{\boldsymbol{\theta}}%

\global\long\def\blambda{\boldsymbol{\lambda}}%

\global\long\def\bgamma{\boldsymbol{\gamma}}%

\global\long\def\bpsi{\boldsymbol{\psi}}%

\global\long\def\bphi{\boldsymbol{\phi}}%

\global\long\def\bpi{\boldsymbol{\pi}}%

\global\long\def\eeta{\boldsymbol{\eta}}%

\global\long\def\bomega{\boldsymbol{\omega}}%

\global\long\def\bepsilon{\boldsymbol{\epsilon}}%

\global\long\def\btau{\boldsymbol{\tau}}%

\global\long\def\bSigma{\gvt{\Sigma}}%

\global\long\def\realset{\mathbb{R}}%

\global\long\def\realn{\realset^{n}}%

\global\long\def\integerset{\mathbb{Z}}%

\global\long\def\natset{\integerset}%

\global\long\def\integer{\integerset}%

\global\long\def\natn{\natset^{n}}%

\global\long\def\rational{\mathbb{Q}}%

\global\long\def\rationaln{\rational^{n}}%

\global\long\def\complexset{\mathbb{C}}%

\global\long\def\comp{\complexset}%

\global\long\def\compl#1{#1^{\text{c}}}%

\global\long\def\and{\cap}%

\global\long\def\compn{\comp^{n}}%

\global\long\def\comb#1#2{\left({#1\atop #2}\right) }%

\global\long\def\nchoosek#1#2{\left({#1\atop #2}\right)}%

\global\long\def\param{\vt w}%

\global\long\def\Param{\Theta}%

\global\long\def\meanparam{\gvt{\mu}}%

\global\long\def\meanmap{\mathbf{m}}%

\global\long\def\logpart{A}%

\global\long\def\simplex{\Delta}%

\global\long\def\simplexn{\simplex^{n}}%

\global\long\def\dirproc{\text{DP}}%

\global\long\def\ggproc{\text{GG}}%

\global\long\def\DP{\text{DP}}%

\global\long\def\ndp{\text{nDP}}%

\global\long\def\hdp{\text{HDP}}%

\global\long\def\gempdf{\text{GEM}}%

\global\long\def\ei{\text{EI}}%

\global\long\def\rfs{\text{RFS}}%

\global\long\def\bernrfs{\text{BernoulliRFS}}%

\global\long\def\poissrfs{\text{PoissonRFS}}%

\global\long\def\grad{\gradient}%
 
\global\long\def\gradient{\nabla}%

\global\long\def\cpr#1#2{\Pr\left(#1\ |\ #2\right)}%

\global\long\def\var{\text{Var}}%

\global\long\def\Var#1{\text{Var}\left[#1\right]}%

\global\long\def\cov{\text{Cov}}%

\global\long\def\Cov#1{\cov\left[ #1 \right]}%

\global\long\def\COV#1#2{\underset{#2}{\cov}\left[ #1 \right]}%

\global\long\def\corr{\text{Corr}}%

\global\long\def\sst{\text{T}}%

\global\long\def\SST{\sst}%

\global\long\def\ess{\mathbb{E}}%

\global\long\def\Ess#1{\ess\left[#1\right]}%

%\newcommandx\ESS[2][usedefault, addprefix=\global, 1=]{\underset{#2}{\ess}\left[#1\right]}%

\global\long\def\fisher{\mathcal{F}}%

\global\long\def\bfield{\mathcal{B}}%
 
\global\long\def\borel{\mathcal{B}}%

\global\long\def\bernpdf{\text{Bernoulli}}%

\global\long\def\betapdf{\text{Beta}}%

\global\long\def\dirpdf{\text{Dir}}%

\global\long\def\gammapdf{\text{Gamma}}%

\global\long\def\gaussden#1#2{\text{Normal}\left(#1, #2 \right) }%

\global\long\def\gauss{\mathbf{N}}%

\global\long\def\gausspdf#1#2#3{\text{Normal}\left( #1 \lcabra{#2, #3}\right) }%

\global\long\def\multpdf{\text{Mult}}%

\global\long\def\poiss{\text{Pois}}%

\global\long\def\poissonpdf{\text{Poisson}}%

\global\long\def\pgpdf{\text{PG}}%

\global\long\def\iwshpdf{\text{InvWish}}%

\global\long\def\nwpdf{\text{NW}}%

\global\long\def\niwpdf{\text{NIW}}%

\global\long\def\studentpdf{\text{Student}}%

\global\long\def\unipdf{\text{Uni}}%

\global\long\def\transp#1{\transpose{#1}}%
 
\global\long\def\transpose#1{#1^{\mathsf{T}}}%

\global\long\def\mgt{\succ}%

\global\long\def\mge{\succeq}%

\global\long\def\idenmat{\mathbf{I}}%

\global\long\def\trace{\mathrm{tr}}%

\begin{abstract}
The study of robustness has received much attention due to its inevitability in data-driven settings where many systems face uncertainty. One such example of concern is Bayesian Optimization (BO), where uncertainty is multi-faceted, yet there only exists a limited number of works dedicated to this direction. In particular, there is the work of \citet{kirschner2020distributionally}, which bridges the existing literature of Distributionally Robust Optimization (DRO) by casting the BO problem from the lens of DRO. While this work is pioneering, it admittedly suffers from various practical shortcomings such as finite contexts assumptions, leaving behind the main question \textit{Can one devise a computationally tractable algorithm for solving this DRO-BO problem}? In this work, we tackle this question to a large degree of generality by considering robustness against data-shift in $\varphi$-divergences, which subsumes many popular choices, such as the $\chi^2$-divergence, Total Variation, and the extant Kullback-Leibler (KL) divergence. We show that the DRO-BO problem in this setting is equivalent to a finite-dimensional optimization problem which, even in the continuous context setting, can be easily implemented with provable sublinear regret bounds. We then show experimentally that our method surpasses existing methods, attesting to the theoretical results.

%\vu{think about the title? \citep{kirschner2020distributionally} has already used "Distributionally Robust Bayesian Optimization"} \vu{I'd suggest "1. Efficient Distributionally Robust Bayesian Optimization} \vu{"2. Provably Efficient Distributionally Robust Bayesian Optimization}
%\vu{ICLR'22. Abstract Sep 29 '21. Paper Oct 06 '21 Amazon internal 17 Sep}
%\vu{We use American-English in this paper}
\end{abstract}

\section{Introduction}
 Bayesian Optimization (BO) [\citenum{Kushner_1964New,Jones_1998Efficient,Srinivas_2010Gaussian,Shahriari_2016Taking,nguyen2020knowing}] allows us to model a black-box function that is expensive to evaluate, in the case where noisy observations are available.  Many important applications of BO correspond to situations where the objective function depends on an additional context parameter [\citenum{Krause_2011Contextual,van2021personalized}], for example in health-care, recommender systems can be used to model information about a certain type of medical domain. BO has naturally found success in a number of scientific domains [\citenum{ueno2016combo,Hernandez_2017Parallel,li2018accelerating_ICDM, gopakumar2018algorithmic, tran2021simulation}] and also a staple in machine learning for the crucial problem of hyperparameter tuning [\citenum{ru2020bayesian,TVO_GP,parker2020provably,perrone2020amazon,wan2022bayesian}]. 

As with all data-driven approaches, BO is prone to cases where the given data \textit{shifts} from the data of interest. While BO models this in the form of Gaussian noise for the inputs to the objective function, the context distribution is assumed to be consistent. This can be problematic, for example in healthcare where patient information shifts over time. This problem exists in the larger domain of operations research under the banner of \textit{distributionally robust optimization} (DRO) [\citenum{scarf1957min}], where one is interested in being \textit{robust} against shifts in the distribution observed. In particular, for a given \textit{distance} between distributions $\mathsf{D}$, DRO studies robustness against adversaries who are allowed to modify the observed distribution $p$ to another distribution in the set:
\begin{align*}
    \braces{q : \mathsf{D}(p,q) \leq \varepsilon},
\end{align*}
for some $\varepsilon > 0$. One can interpret this as a ball of radius $\varepsilon$ for the given choice of $\mathsf{D}$ and the adversary perturbs the observed distribution $p$ to $q$ where $\varepsilon$ is a form of ``budget''. 

Distributional shift is a topical problem in machine learning and the results of DRO have been specialized in the context of supervised learning [\citenum{duchi2013local,duchi2016statistics,duchi2019variance,cranko2020generalised,blanchet2019robust,gao2017wasserstein,husain2020distributional}], reinforcement learning [\citenum{hou2020robust}] and Bayesian learning [\citenum{shapiro2021bayesian}], as examples. One of the main challenges however is that the DRO is typically intractable since in the general setting of continuous contexts, involves an infinite dimensional constrained optimization problem. The choice of $\mathsf{D}$ is crucial here as various choices such as the Wasserstein distance [\citenum{blanchet2019quantifying, blanchet2019robust,cranko2018lipschitz,shafieezadeh2019regularization}], Maximum Mean Discrepancy (MMD) [\citenum{staib2019distributionally}] and $\varphi$-divergences \footnote{as known as $f$-divergences in the literature} [\citenum{duchi2013local,duchi2016statistics}] allow for computationally tractable regimes. In particular, these specific choices of $\mathsf{D}$ have shown intimate links between regularization [\citenum{husain2020distributional}] which is a conceptually central topic of machine learning.

More recently however, DRO has been studied for the BO setting in \citet{kirschner2020distributionally}, which as one would expect, leads to a complicated minimax problem, which causes a computational burden practically speaking. \citet{kirschner2020distributionally} makes the first step and casts the formal problem however develops an algorithm only in the case where $\mathsf{D}$ has been selected as the MMD. While, this work makes the first step and conceptualizes the problem of distributional shifts in context for BO, there are two main practical short-comings. Firstly, the algorithm is developed specifically to the MMD, which is easily computed, however cannot be replaced by another choice of $\mathsf{D}$ whose closed form is not readily accessible with samples such as the $\varphi$-divergence. Secondly, the algorithm is only tractable when the contexts are finite since at every iteration of BO, it requires solving an $M$-dimensional problem where $M$ is the number of contexts. 

The main question that remains is, \textit{can we devise an algorithm that is computationally tractable for tackling the DRO-BO setting}? We answer this question to a large degree of generality by considering distributional shifts against $\varphi$-divergences - a large family of divergences consisting of the extant Kullback-Leibler (KL) divergence, Total Variation (TV) and $\chi^2$-divergence, among others. In particular, we exploit existing advances made in the large literature of DRO to show that the BO objective in this setting for any choice of $\varphi$-divergence yields a computationally tractable algorithm, even for the case of continuous contexts.  We also present a robust regret analysis that illustrates a sublinear regret. Finally, we show, along with computational tractability, that our method is empirically superior on standard datasets against several baselines including that of \citet{kirschner2020distributionally}. In summary, our main contributions are
\begin{enumerate}
    \item A theoretical result showing that the  minimax distributionally robust BO objective with respect to $\varphi$ divergences is equivalent to a single minimization problem. 
    \item An efficient algorithm, that works in the continuous context regime, for the specific cases of the $\chi^2$-divergence and TV distance, which admits a conceptually interesting relationship to regularization of BO.
    \item A regret analysis that specifically informs how we can choose the DRO $\varepsilon$-budget to attain sublinear regret.
\end{enumerate}

\section{Related Work}
\vspace{-0.2cm}
Due to the multifaceted nature of our contribution, we discuss two streams of related literature, one relating to studies of robustness in Bayesian Optimization (BO) and one relating to advances in Distributionally Robust Optimization (DRO).

In terms of BO, the work most closest to ours is \citet{kirschner2020distributionally} which casts the distributionally robust optimization problem over contexts. In particular, the work shows how the DRO objective for any choice of divergence $\mathsf{D}$ can be cast, which is exactly what we build off. The main drawback of this method however is the limited practical setting due to the expensive inner optimization, which heavily relies on the MMD, and therefore cannot generalize easily to other divergences that are not available in closed forms. Our work in comparison, holds for a much more general class of divergences, and admits a practical algorithm that involves a finite dimensional optimization problem. In particular, we derive the result when $\mathsf{D}$ is chosen to be the $\chi^2$-divergence which we show performs the best empirically. This choice of divergence has been studied in the related problem of Bayesian quadrature [\citenum{DRBQO}], and similarly illustrated strong performance, complimenting our results. There also exists work of BO that aim to be robust by modelling adversaries through noise, point estimates or non-cooperative games [\citenum{Nogueira_2016Unscented, martinez2018practical, beland2017bayesian,oliveira2019bayesian, sessa2019no}]. The main difference between our work and theirs is that the notion of robustness we tackle is at the \textit{distributional} level. Another similar work to ours is that of \citet{tay2022efficient} which considers approximating DRO-BO using Taylor expansions based on the sensitivity of the function. In some cases, the results coincide with ours however their result must account for an approximation error in general. Furthermore, an open problem as stated in their work is to solve the DRO-BO problem for continuous context domains, which is precisely one of the advantages of our work.

%\hisham{I will add papers here about adversarial BO and then discuss papers that learn adversarial attacks using BO}. \vu{good idea. we dont compare with these adversarial BO in the experiments. Therefore, we can emphasize by saying it is related, but different...}

From the perspective of DRO, our work essentially is an extension of \citet{duchi2013local,duchi2016statistics} which develops results that connect $\varphi$-divergence DRO to variance regularization. In particular, they assume $\varphi$ admits a continuous second derivative, which allows them to connect the $\varphi$-divergence to the $\chi^2$-divergence and consequently forms a general connection to constrained variance. While the work is pioneering, this assumption leaves out important $\varphi$-divergences such as the Total Variation (TV) - a choice of divergence which we illustrate performs well in comparison to standard baselines in BO. At the technical level, our derivations are  similar to \citet{ahmadi2012entropic} however our result, to the best of our knowledge, is the first such work that develops it in the context of BO. In particular, our results for the Total Variation and $\chi^2$-divergence show that variance is a key penalty in ensuring robustness which is a well-known phenomena existing in the realm of machine learning [\citenum{duchi2013local,duchi2019variance,cranko2020generalised,husain2020distributional,abadeh2015distributionally}].

\section{Preliminaries} \label{sec:preliminary}
\paragraph{Bayesian Optimization} We consider optimizing a \textit{black-box} function, $f: \mathcal{X} \to \mathbb{R}$ with respect to the \textit{input} space $\mathcal{X} \subseteq \mathbb{R}^d$. As a black-box function, we do not have access to $f$ directly however receive input in a sequential manner: at time step $t$, the learner chooses some input $\bx_t \in \mathcal{X}$ and observes the \textit{reward} $y_t= f(\bx_t) + \eta_t$ where the noise $\eta_t \sim \mathcal{N}(0,\sigma^2_f)$ and $\sigma^2_f$ is the output noise variance. Therefore, the goal is to optimize 
\begin{align*}
    \sup_{\bx \in \mathcal{X}} f(\bx).
\end{align*}
Additional to the input space $\mathcal{X}$, we introduce the \textit{context} spaces $\mathcal{C}$, which we assume to be compact. These spaces are assumed to be  separable completely metrizable topological spaces.\footnote{We remark that this is an extremely mild condition, satisfied by the large majority of considered examples.} We have a reward function, $f: \mathcal{X} \times \mathcal{C} \to \mathbb{R}$ which we are interested in optimizing with respect to $\mathcal{X}$. Similar to sequential optimization, at time step $t$ the learner chooses some input $\bx_t \in \mathcal{X}$ and receives a context $c_t \in \mathcal{C}$ and $f(\bx_t,c_t) + \eta_t$. Here, the learner can not choose a context $c_t$, but receive it from the environment. Given the context information, the objective function is written as 
\begin{align*}
    \sup_{\bx \in \mathcal{X}} \E_{c \sim p}[f(\bx,c)], 
\end{align*}
where $p$ is a probability distribution over contexts.

% The kernel $k(\bx,\bx')$ can be thought of as a similarity measure relating $f(\bx)$ and $f(\bx')$. 
% assuming a zero mean $m(\bx)=0$, we have:
% \begin{align*}
% \left[\begin{array}{c}
% \boldsymbol{f}\\
% f_{*}
% \end{array}\right] & \sim\mathcal{N}\left(0,\left[\begin{array}{cc}
% \bK & \bk_{*}^{T}\\
% \bk_{*} & k_{**}
% \end{array}\right]\right)\label{eq:p(f|f*)}
% \end{align*}
%  where $k_{**}=k\left(\bx_{*},\bx_{*}\right)$, $\bk_{*}=[k\left(\bx_{*},\bx_{i}\right)]_{\forall i\le N}$
% and $\bK=\left[k\left(\bx_{i},\bx_{j}\right)\right]_{\forall i,j\le N}$.

\paragraph{Gaussian Processes} We follow a popular choice in BO [\citenum{Shahriari_2016Taking}] to use GP as a surrogate model for optimizing  $f$. A GP [\citenum{Rasmussen_2006gaussian}] defines a probability distribution
over functions $f$ under the assumption that any subset of points
$\left\{ \bx_{i},f(\bx_{i})\right\} $ is normally distributed. Formally,
this is denoted as:
\begin{align*}
f(\bx)\sim \text{GP}\left(m(\bx),k(\bx,\bx')\right),
\end{align*}
where $m\left(\bx\right)$ and $k\left(\bx,\bx'\right)$ are the mean
and covariance functions, given by $m(\bx)=\mathbb{E}\left[f(\bx)\right]$
and $k(\bx,\bx')=\mathbb{E}\left[(f(\bx)-m(\bx))(f(\bx')-m(\bx'))^{T}\right]$. For predicting $f_{*}=f\left(\bx_{*}\right)$ at a new data point $\bx_{*}$,  
the conditional probability
follows a univariate Gaussian distribution as $p\bigl(f_{*}\mid \bx_*, [\bx_1...\bx_N],[y_1,...y_N] \bigr)\sim\mathcal{N}\left(\mu\left(\bx_{*}\right),\sigma^{2}\left(\bx_{*}\right)\right)$.
Its mean and variance are given by:
\begin{minipage}{0.49\textwidth}
\begin{align}
\mu\left(\bx_{*}\right)= & \mathbf{k}_{*,N}\mathbf{K}_{N,N}^{-1}\mathbf{y}, \label{eq:gp_mean}
\end{align}
\end{minipage}
\begin{minipage}{0.49\textwidth}
\begin{align}
\sigma^{2}\left(\bx_{*}\right)= & k_{**}-\mathbf{k}_{*,N}\mathbf{K}_{N,N}^{-1}\mathbf{k}_{*,N}^{T}\label{eq:gp_var}
\end{align}
\end{minipage}

where $k_{**}=k\left(\bx_{*},\bx_{*}\right)$, $\bk_{*,N}=[k\left(\bx_{*},\bx_{i}\right)]_{\forall i\le N}$ and $\bK_{N,N}=\left[k\left(\bx_{i},\bx_{j}\right)\right]_{\forall i,j\le N}$.
 As GPs give full uncertainty information with any prediction, they
provide a flexible nonparametric prior for Bayesian optimization.
We refer to \citet{Rasmussen_2006gaussian}
for further details on GPs.

\paragraph{Distributional Robustness} 
Let $\Delta(\mathcal{C})$ denote the set of probability distributions over $\mathcal{C}$. A \textit{divergence} between distributions $\mathsf{D}: \Delta(\mathcal{C}) \times \Delta(\mathcal{C}) \to \mathbb{R}$ is a dissimilarity measure that satisfies $\Delta(p,q) \geq 0$ with equality if and only if $p = q$ for $p,q \in \Delta(\mathcal{C})$. For a function, $h: \mathcal{C} \to \mathbb{R}$, base probability measure $p \in \Delta(\mathcal{C})$, the central concern of Distributionally Robust Optimization (DRO) [\citenum{ben2013robust,rahimian2019distributionally,bennouna2021learning}] is to compute
\begin{align}
\label{background:dro-eq1}
    \sup_{q \in B_{\varepsilon,\mathsf{D}}(p)} \E_{q(c)}[h(c)],
\end{align}
where $B_{\varepsilon,\mathsf{D}}(p) = \braces{q \in \Delta(\mathcal{C}) :  \mathsf{D}(p,q) \leq \varepsilon}$, is ball of distributions $q$ that are $\varepsilon$ away from $p$ with respect to the divergence $\mathsf{D}$. The objective in Eq. \eqref{background:dro-eq1} is intractable, especially in setting where $\mathcal{C}$ is continuous as it amounts to a constrained infinite dimensional optimization problem. It is also clear that the choice of $\mathsf{D}$ is crucial for both computational and conceptual purposes. The vast majority of choices typically include the Wasserstein due to the transportation-theoretic interpretation and with a large portion of existing literature finding connections to Lipschitz regularization [\citenum{blanchet2019quantifying, blanchet2019robust,cranko2018lipschitz,shafieezadeh2019regularization}]. Other choices where they have been studied in the supervised learning setting include the Maximum Mean Discrepancy (MMD) [\citenum{staib2019distributionally}] and $\varphi$-divergences [\citenum{duchi2013local,duchi2016statistics}].

\paragraph{Distributionally Robust Bayesian Optimization} Recently, the notion of DRO has been applied to BO [\citenum{kirschner2020distributionally,tay2022efficient}], who consider robustness with respect to shifts in the context space and therefore are interested in solving
\begin{align*}
    \sup_{\bx \in \mathcal{X}} \inf_{q \in B_{\varepsilon, D}(p)} \E_{c \sim q}[f(\bx,c)],
\end{align*}
where $p$ is the reference distribution. This objective becomes significantly more difficult to deal with since not only does it involve a constrained and possibly infinite dimensional optimization problem however also involves a minimax which can cause instability issues if solved iteratively. 

% \citet{kirschner2020distributionally} address the objective directly using this formulation of the MMD. 
%Kirschner \textit{et al.}
\citet{kirschner2020distributionally} tackle these problems by letting $\mathsf{D}$ be the kernel Maximum Mean Discrepancy (MMD), which is a popular choice of discrepancy motivated by kernel mean embeddings [\citenum{gretton2012kernel}]. In particular, the MMD can be efficiently estimated in $O(n^2)$ where $n$ is the number of samples. Naturally, this has two main drawbacks: The first is that it is still computationally expensive since one is required to solve two optimization problems, which can lead to instability and secondly, the resulting algorithm is limited to the scheme where the number of contexts is finite. In our work, we consider $\mathsf{D}$ to be a $\phi$-divergence, which includes the Total Variance, $\chi^2$ and Kullback-Leibler (KL) divergence and furthermore show that minmax objective can be reduced to a single maximum optimization problem which resolves both the instability and finiteness assumption. In particular, we also present a similar analysis, showing that the robust regret decays sublinearly for the right choices of radii.

\section{$\phi$-Robust Bayesian Optimization}
In this section, we present the main result on distributionally robustness when applied to BO using $\varphi$-divergence. Therefore, we begin by defining this key quantity.
\begin{definition}[$\phi$-divergence]
Let $\phi: \mathbb{R} \to (-\infty,\infty]$ be a convex, lower semi-continuous function such that $\phi(1) = 0$. The $\phi$-divergence between $p,q \in \Delta(\mathcal{C})$ is defined as
\begin{align*}
    \mathsf{D}_{\phi}(p,q) = \E_{q(c)}\left[\phi\bracket{\frac{dp}{dq}(c)} \right],
\end{align*}
where $dp/dq$ is the Radon-Nikodym derivative if $p \ll q$ and $\mathsf{D}_{\varphi}(p,q) = +\infty$ otherwise.
\end{definition}
Popular choices of the convex function $\phi$ include $\phi(u) = (u-1)^2$ which yields the $\chi^2$ and, $\phi(u) = \card{u-1}$, $\phi(u) = u \log u$ which correspond to the $\chi^2$ and KL divergences respectively. At any time step $t \geq 1$, we consider distributional shifts with respect to an $\phi$-divergence for any choice of $\phi$ and therefore relevantly define the DRO ball as %as with respect to the $\varphi$-divergence
\begin{align*}
    B_{\phi}^{t}(p_t) := \braces{q \in \Delta(\mathcal{C}) : \mathsf{D}_{\phi}(q,p_t) \leq \varepsilon_t },
\end{align*}
where $p_t = \frac{1}{t} \sum_{s=1}^t \delta_{c_s}$ is the reference distribution and $\varepsilon_t$ is the distributionally robust radius chosen at time $t$. We remark that for our results, the choice of $p_t$ is flexible and can be chosen based on the specific domain application. The $\varphi$ divergence, as noted from the definition above, is only defined finitely when the measures $p,q$ are absolutely continuous to each other and there is regarded as a \textit{strong} divergence in comparison to the Maximum Mean Discrepancy (MMD), which is utilized in \citet{kirschner2020distributionally}. The main consequence of this property is that the geometry of the ball $B_{\varphi}^t$ would differ based on the choice of $\varphi$-divergence. The $\varphi$-divergence is a very popular choice for defining this ball in previous studies of DRO in the context of supervised learning due to the connections and links it has found to variance regularization [\citenum{duchi2013local,duchi2016statistics,duchi2019variance}].

We will exploit various properties of the $\varphi$-divergence to derive a result that reaps the benefits of this choice such as a reduced optimization problem - a development that does not currently exist for the MMD [\citenum{kirschner2020distributionally}]. We first define the convex conjugate of $\varphi$ as $\phi^{\star}(u) = \sup_{u' \in \operatorname{dom}_{\phi} } \bracket{u\cdot u' - \phi(u')}$, which we note is a standard function that is readily available in closed form for many choices of $\varphi$.
\begin{theorem}\label{theorem:main}
Let $\phi: \mathbb{R} \to (-\infty, \infty]$ be a convex lower semicontinuous mapping such that $\phi(1) = 0$. Let $f$ be measurable and bounded. For any $\epsilon > 0$, it holds that
\begin{align*}
    \sup_{\bx \in \mathcal{X}} \inf_{q \in B_{\phi}^{t}(p)} \E_{c \sim q}[f(\bx,c)] = 
    \sup_{\bx \in \mathcal{X},  \lambda \geq 0, b \in \mathbb{R}} \bracket{ b - \lambda \epsilon_t - \lambda \E_{p_{t}(c)}\left[ \varphi^{\star} \bracket{\frac{b - f(\bx,c)}{\lambda} }\right]  }.
\end{align*}
\end{theorem}
\textbf{Proof (Sketch)} 
The proof begins by rewriting the constraint over the $\varphi$-divergence constrained ball with the use of Lagrangian multipliers. Using existing identities for $f$-divergences, a minimax swap yields a two-dimensional optimization problem, over $\lambda\geq 0$ and $b \in \mathbb{R}$.

We remark that similar results exist for other areas such as supervised learning [\citenum{shapiro2017distributionally}], robust optimization [\citenum{ben2013robust}] and certifying robust radii [\citenum{dvijotham2020framework}]. However this is, to the best of our knowledge, the first development when applied to optimizing expensive black-box functions, the case of BO. The above Theorem  is practically compelling for three main reasons. First, one can note that compared to the left-hand side, the result converts this into a single optimization (max) over three variables, where two of the variables are $1$-dimensional, reducing the computational burden significantly. Secondly, the notoriously difficult max-min problem becomes only a max, leaving behind instabilities one would encounter with the former objective. Finally, the result makes very mild assumptions on the context parameter space $\mathcal{C}$, allowing infinite spaces to be chosen, which is one of the challenges for existing BO advancements. We show that for specific choices of $\phi$, the optimization over $b$ and even $\lambda$ can be expressed in closed form and thus simplified. All proofs for the following examples can be found in the Appendix Section \ref{supp-formal}.

\begin{example}[$\chi^2$-divergence] \label{example_chi2}
Let $\phi(u) = (u-1)^2$, then for any measurable and bounded $f$ we have for any choice of $\varepsilon_t$
\begin{align*}
    &\sup_{\bx \in \mathcal{X}} \inf_{q \in B_{\phi}^{t}(p_t)} \E_{c \sim q}[f(\bx,c)] = \sup_{\bx \in \mathcal{X}} \bracket{ \E_{p_t(c)}[f(\bx,c)] 
    - \sqrt{\varepsilon_t \cdot \operatorname{Var}_{p_t(c)}[f(\bx,c)] }}.
\end{align*}
\end{example}
The above example can be easily implemented as it involves the same optimization problem however now appended with a variance term. Furthermore, this objective admits a compelling conceptual insight which is that, by enforcing a penalty in the form of variance, one attains robustness. The idea that regularization provides guidance to robustness or generalization is well-founded in machine learning more generally for example in supervised learning [\citenum{duchi2013local,duchi2016statistics}]. We remark that this penalty and its relationship to $\chi^2$-divergence has been developed in the similar yet related problem of Bayesian quadrature [\citenum{DRBQO}]. Moreover, it can be shown that if $\varphi$ is twice differentiable then $\mathsf{D}_{\varphi}$ can be approximated by the $\chi^2$-divergence via Taylor series, which makes $\chi^2$-divergence a centrally appealing choice for studying robustness. We now derive the result for a popular choice of $\varphi$ that is not differentiable.
\vspace{-0.75em}
\begin{example}[Total Variation] \label{example_tv}
Let $\varphi(u) = \card{u-1}$, then for any measurable and bounded $f$ we have for any choice of $\varepsilon_t$ %$$
\begin{align*}
   % & \nonumber \\
    &\sup_{\bx \in \mathcal{X}} \inf_{q \in B_{\phi}^{t}(p_t)} \E_{c \sim q}[f(\bx,c)] = \sup_{\bx \in \mathcal{X}} \bracket{ \E_{p_t(c)}[f(\bx,c)] - \frac{\varepsilon_t}{2}\bracket{\sup_{c \in \mathcal{C}}f(\bx,c) - \inf_{c \in \mathcal{C}}f(\bx,c)  } }.
\end{align*}
\end{example}
Similar to the $\chi^2$-case, the result here admits a variance-like term in the form of the difference between the maximal and minimal elements. We remark that such a result is conceptually interesting since both losses admit an objective that resembles a mean-variance which is a natural concept in ML, but advocates for it from the perspective of distributional robustness. This result exists for the supervised learning in \citet{duchi2019variance} however is completely novel for BO and also holds for a choice of non-differentiable $\varphi$, hinting at the deeper connection between $\varphi$-divergence DRO and variance regularization.

%We build a Gaussian process as the surrogate model for  Bayesian optimization.

\subsection{Optimization with the GP Surrogate}

To handle the distributional robustness, we have rewritten the objective function using $\phi$ divergences in Theorem \ref{theorem:main}.
In DRBO setting, we sequentially select a next point $\bx_t$ for querying a black-box function. Given the observed context $c_t \sim q$ coming from the environment, we evaluate the black-box function and observe the output as $y_t = f(\bx_t, c_t) + \eta_t$ where the noise $\eta_t \sim \mathcal{N}(0,\sigma^2_f)$ and $\sigma^2_f$ is the noise variance.
\begin{wrapfigure}{R}{0.55\textwidth}
\begin{minipage}{0.55\textwidth}
%\vspace{-2pt}
\begin{algorithm}[H]
    \caption{DRBO with $\varphi$-divergence \label{alg:drbo}}
	\begin{algorithmic}[1]
		\STATE {\bfseries Input:}  Max iteration $T$, initial data $D_0$,    $\eta$
	
    	%\STATE Total Variation: $g(\bx,\bc) = \E_{p(c)}[f(\bx,c)] - \frac{\varepsilon}{2}\bracket{\sup_{c \in \mathcal{C}}f(\bx,c) - \inf_{c \in \mathcal{C}}f(\bx,c)  } $
        %\STATE $\chi^2$-divergence: $g(\bx,\bc) = \E_{p(c)}[f(\bx,c)] - \sqrt{\varepsilon \cdot \operatorname{Var}_{p(c)}[f(\bx,c)] }$
        
        \FOR{$t=1, \dots, T$}
            
        	\STATE Fit and estimate GP hyperparameter given $D_{t-1}$
        	%\STATE Select a next input $\bx_{t}=\arg\max\alpha^{\chi^2}(\bx)$
        \STATE Select a next input $\bx_{t}=\arg\max\alpha(\bx)$
        	\STATE \underline{$\chi^2$-divergence}: $\alpha(\bx) := \alpha^{\chi^2} (\bx )$ from Eq. (\ref{eq:alpha_x2})
        	\STATE \underline{Total Variation}: $\alpha(\bx) := \alpha^{TV} (\bx )$ from Eq. (\ref{eq:alpha_tv})
        	\STATE Observe a context $c_t \sim q$
        	\STATE Evaluate the black-box $y_t = f(\bx_t,c_t) + \eta_t$
        	\STATE Augment $D_{t}=D_{t-1}\cup\left(\bx_{t},c_t,y_{t} \right)$
    	\ENDFOR
    	
		%\STATE {\bfseries Output:}  $\arg\max_{\bx\in\mathcal{X}}\mu_{T}\left(\bx\mid {D_{T}}\right)$
   
	\end{algorithmic}
\end{algorithm}
\end{minipage}
\end{wrapfigure}

%\begin{figure}[htb]

%\begin{minipage}{0.6\textwidth}

%\end{minipage}
%\end{figure}

As a common practice in BO, at the iteration $t$, we model the GP surrogate model using the observed data $\{\bx_i,y_i\}_{i=1}^{t-1}$ and make a decision by maximizing the acquisition function which is build on top of the GP surrogate: 
\begin{align*}
\bx_t = \arg \max_{\bx \in \mathcal{X} } \alpha(\bx).
\end{align*}
While our method is not restricted to the form of the acquisition function, for convenience in the theoretical analysis, we follow the GP-UCB [\citenum{Srinivas_2010Gaussian}]. Given the GP predictive mean and variance from Eqs. (\ref{eq:gp_mean},\ref{eq:gp_var}), we have the acquisition function for the $\chi^2$ in Example \ref{example_chi2} as follows:
\begin{align} \label{eq:alpha_x2}
    \alpha^{\chi^2}(\bx) :=&  \frac{1}{|C|} \sum_{c} \left[ \mu_t(\bx,c) + \sqrt{\beta_t}\sigma_t(\bx,c) \right] -\sqrt{ \frac{\epsilon_t}{|C|} \sum_{c} \big( \mu_t(\bx,c) - \bar{\mu_t} \big)^2 }
\end{align}
where $\beta_t$ is a explore-exploit hyperparameter defined in \citet{Srinivas_2010Gaussian}, $\bar{\mu_t} = \frac{1}{|C|} \sum_{c} \mu_t(\bx,c) $ and $c \sim q$ can be generated in a one dimensional space to approximate the expectation and the variance. In the experiment, we select $q$ as the uniform distribution, but it is not restricted to.  Similarly, an acquisition function for Total Variation in Example \ref{example_tv} is written as
\begin{align} \label{eq:alpha_tv}
    \alpha^{TV}( \bx ) :=& \frac{1}{|C|} \sum_{c} \left[ \mu_t(\bx,c) + \sqrt{\beta_t}\sigma_t(\bx,c)  \right] -\frac{\epsilon_t}{2} \big( \max \mu_t(\bx,c) - \min \mu_t(\bx,c) \big).
\end{align}
We summarize all computational steps in Algorithm \ref{alg:drbo}.

\paragraph{Computational Efficiency against MMD.}

We make an important remark that since we do not require our context space to be finite, our implementation scales only linearly with the number of context samples $|C|$ drawing from $q$. This allows us to discretize our space and draw as many context samples as required while only paying a linear price.
On the other hand, the MMD [\citenum{kirschner2020distributionally}] at every iteration of $t$ requires solving an $|C|$-dimensional constraint optimization problem that has no closed form solution. We refer to Section \ref{sec_experiment_computational_comparison} for the empirical comparison.

\begin{figure*}
\vspace{-1pt}
\begin{subfigure}[b]{1\textwidth}
\includegraphics[width=0.34\textwidth]{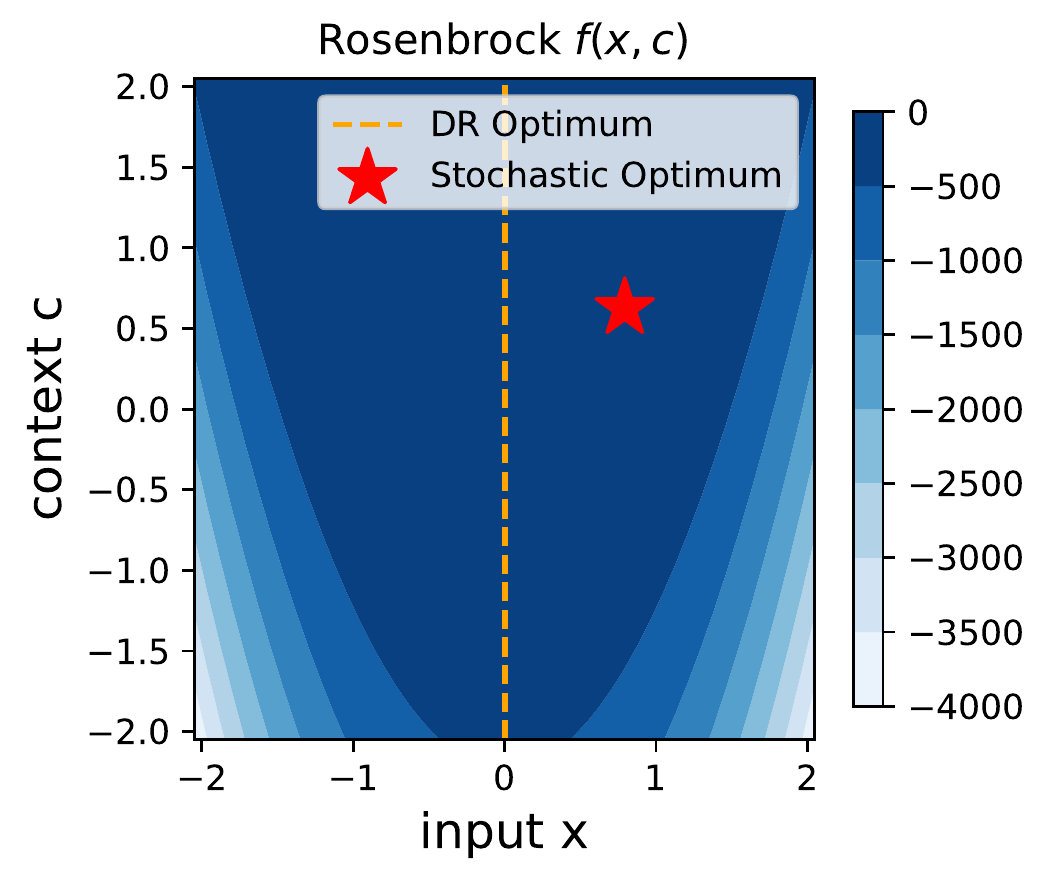}
\includegraphics[width=0.30\textwidth]{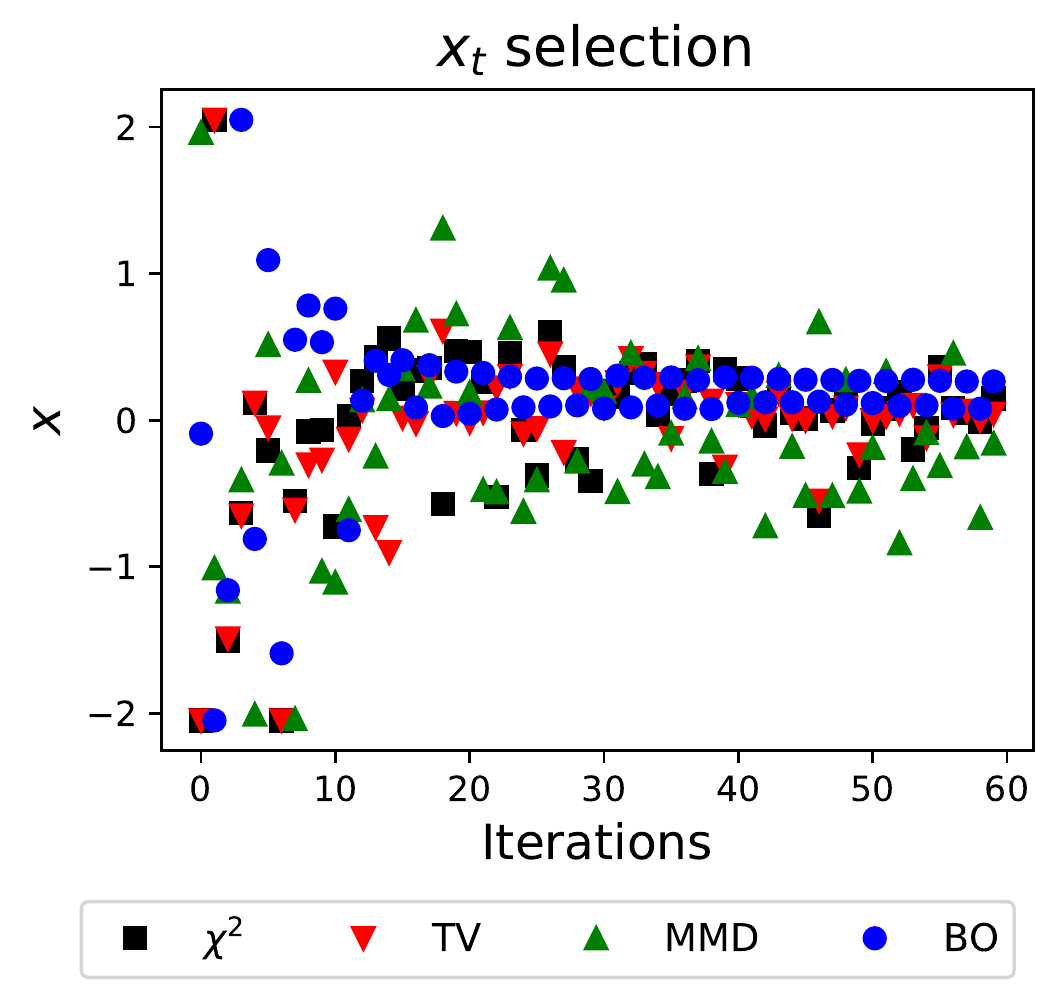}
\includegraphics[width=0.33\textwidth]{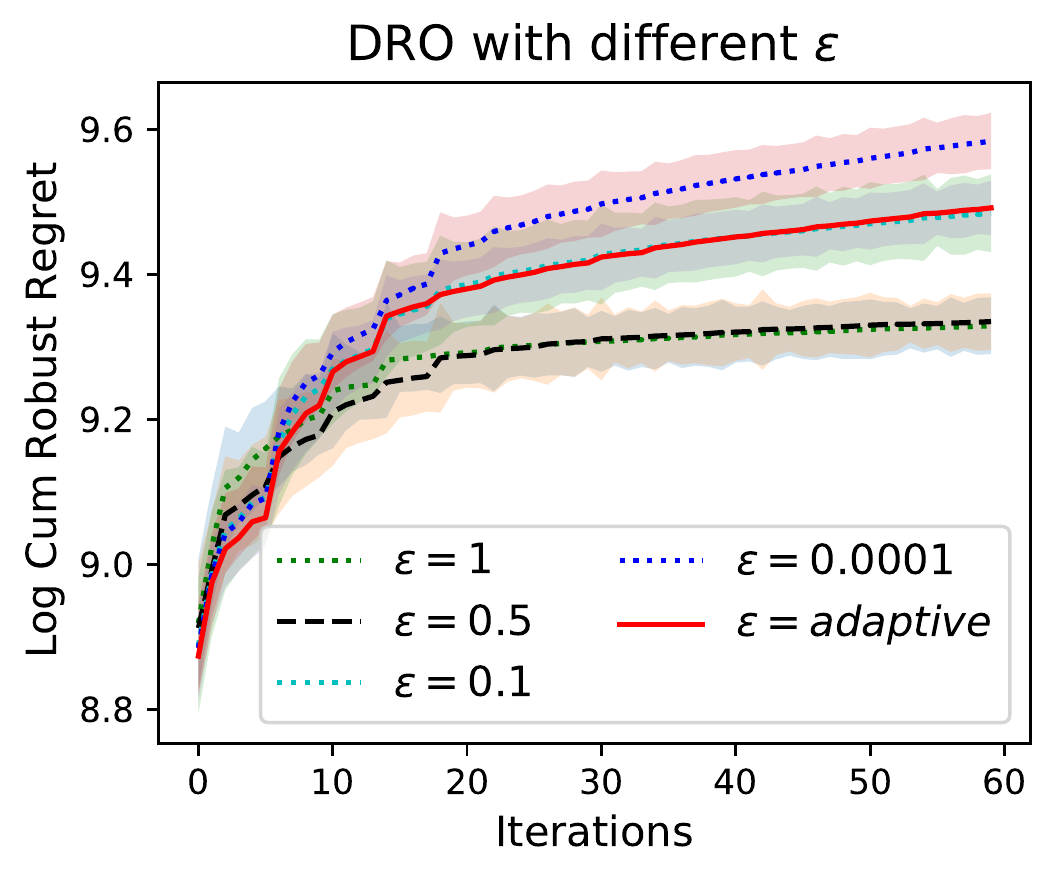}
\vspace{-6pt}
\caption{Stochastic and DRO solutions are different. Our method using $\epsilon=\{0.5,1\}$ result in the best performance.} \label{fig:sto_dro_different}
\end{subfigure}
\vspace{-6pt}
\begin{subfigure}[b]{1\textwidth}
\includegraphics[width=0.32\textwidth]{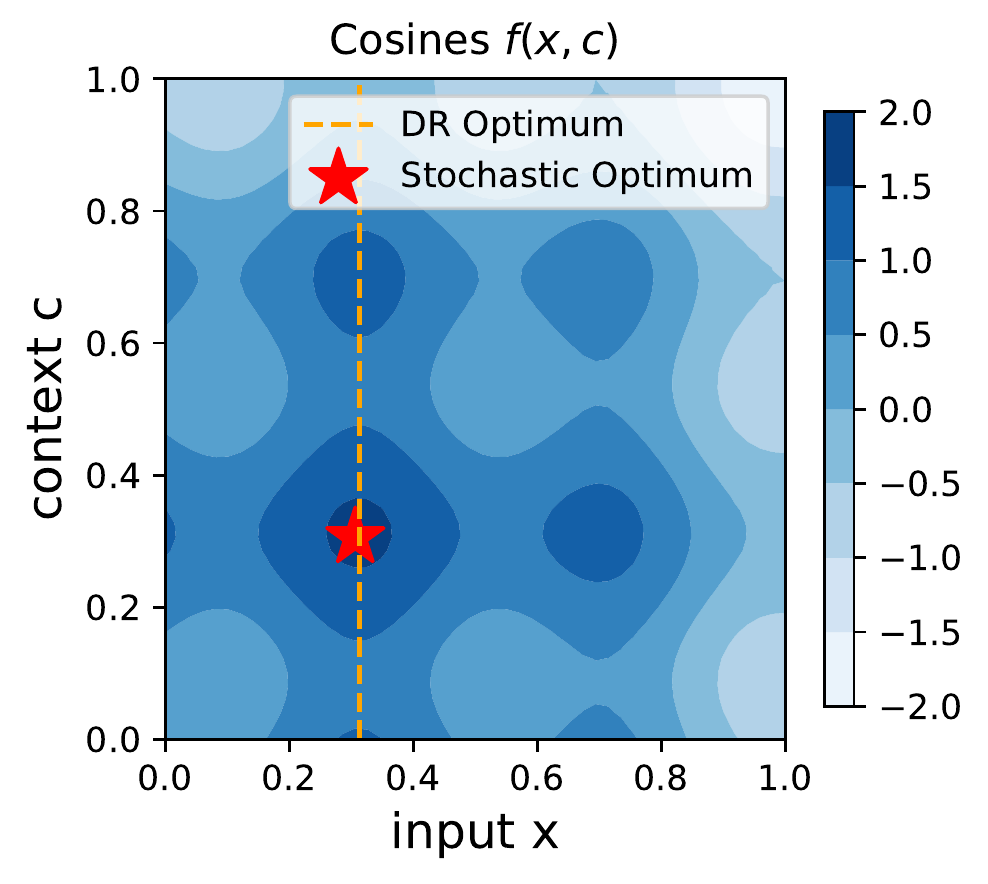}
\includegraphics[width=0.30\textwidth]{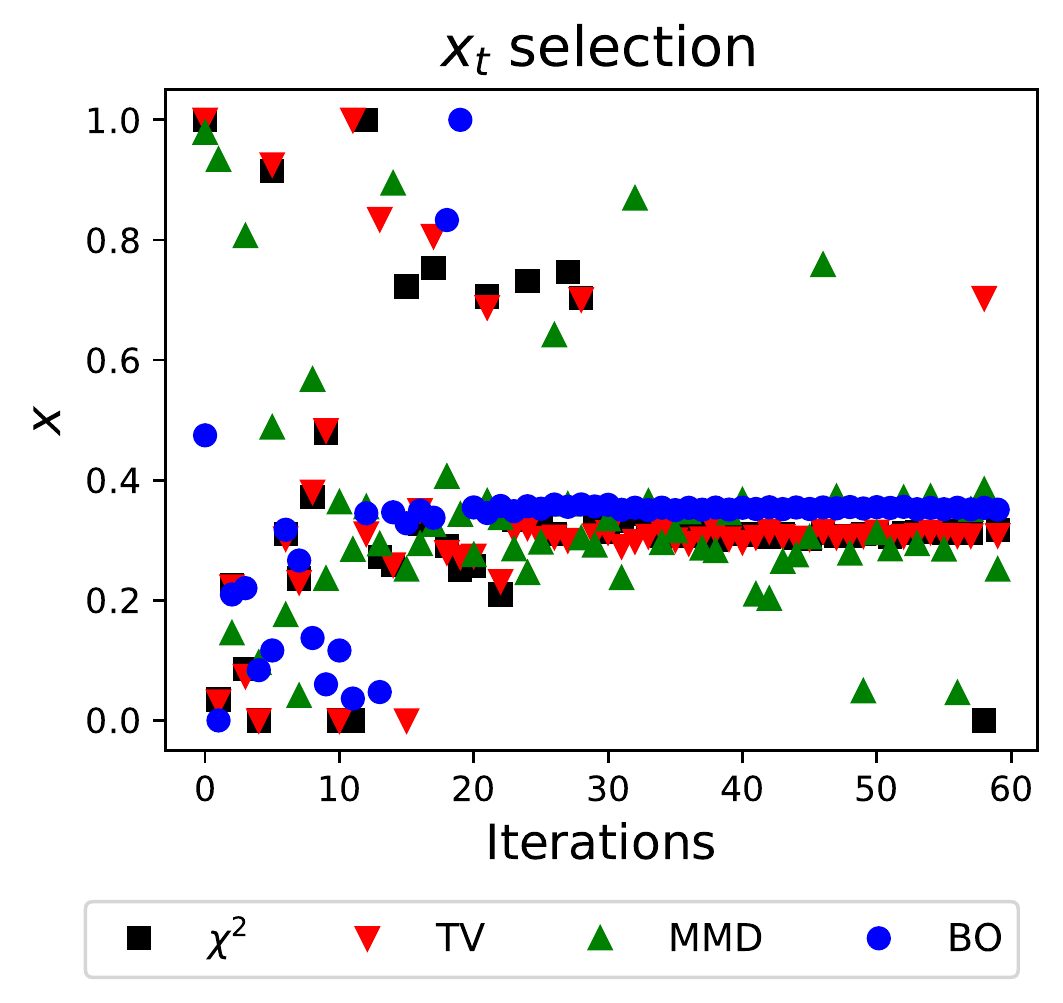}
\includegraphics[width=0.33\textwidth]{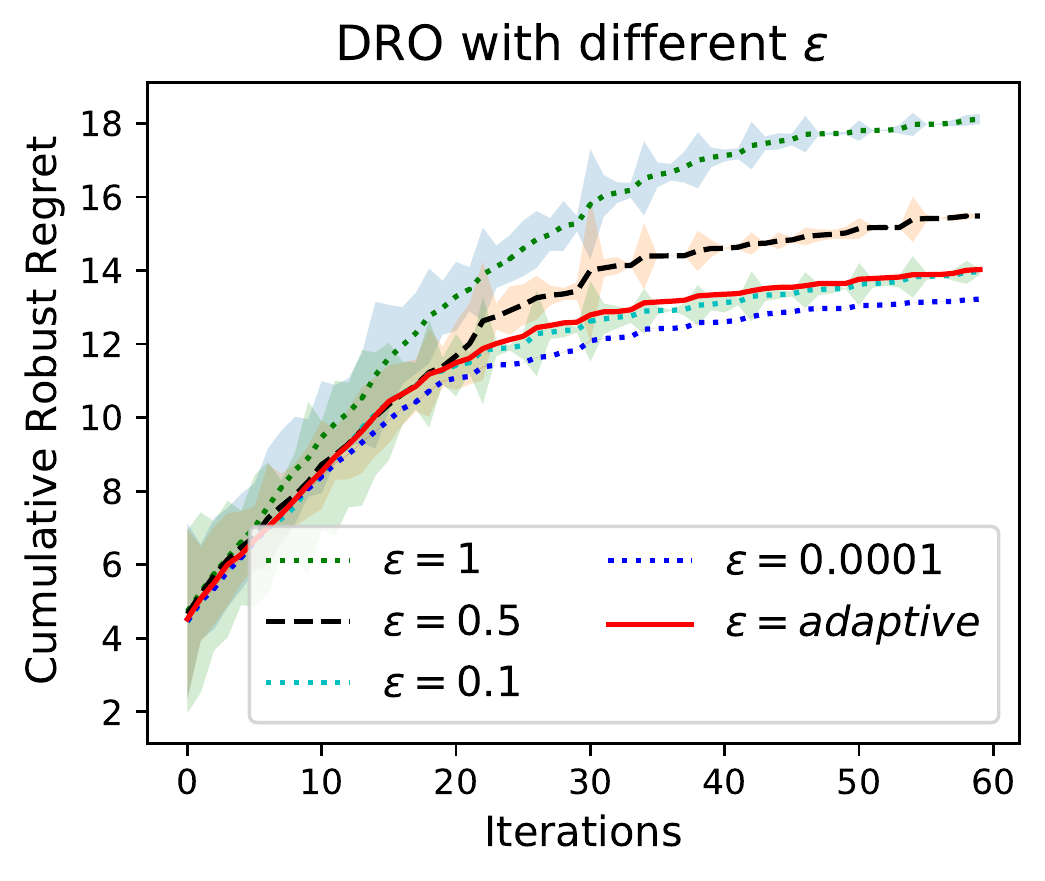}
\vspace{-6pt}
\caption{Stochastic and DRO solutions are coincide. Our method with $\epsilon \rightarrow 0$ is the best.} \label{fig:sto_dro_same}
\end{subfigure}
\vspace{-8pt}
\caption{Two settings in DRO when the stochastic solution and robust solution are different (\textit{top}) and identical (\textit{bottom}). \textit{Left}: original function $f(\bx,c)$. \textit{Middle}: selection of input  $\bx_t$  over iterations. \textit{Right}: performance with different $\epsilon$.} \label{fig:ablation_study} % Best viewed in color.
\vspace{-7pt}
\end{figure*}

\subsection{Convergence Analysis} \label{sec:convergence_analysis}
One of the main advantages of \citet{kirschner2020distributionally} is the choice of MMD makes the regret analysis simpler due to the nice structure and properties of MMD. In particular, the MMD is well-celebrated for a $O(t^{-1/2})$ convergence where no such results exist for $\varphi$-divergences. However, using Theorem \ref{theorem:main}, we can show a regret bound for the Total Variation with a simple boundedness assumption and show how one can extend this result to other $\varphi$-divergences. 
We begin by defining the \textit{robust regret}, $R_T$, with $\varphi$-divergence balls:
\begin{align}
    R_T(\varphi) = \sum_{t=1}^T \inf_{q \in B_{ \phi}^t} \E_{q(c)}[f(\bx_t^{\ast}, c)] - \inf_{q \in B_{\phi}^{t}} \E_{q(c)}[f(\bx_t,c)],  \label{eq:robust_regret}
\end{align}
where $\bx_t^{\ast} = \arg\max_{\bx \in \mathcal{X}}\inf_{q \in B_{\varepsilon, \phi}^t}\E_{q(c)}[f(\bx,c)]$.  We use $\bK_t$ to denote the generated kernel matrix from dataset $D_t = \braces{(\bx_i,c_i)}_{i=1}^t \subset \mathcal{X} \times \mathcal{C}$. 
%In particular, we use $D_t = \braces{(\bx_i,c_i)}_{i=1}^t$ to denote the dataset at round $t$. 
we now introduce a standard quantity in regret analysis in BO is the \textit{maximum information gain}: $\gamma_t = \max_{D \subset \mathcal{X} \times \mathcal{C} : \card{D} = t}\log \det \bracket{ \idenmat_t + \sigma^{-2}_f \bK_t}$ where $\bK_t=\left[k\left([\bx_{i},c_i],[\bx_{j},c_j]\right)\right]_{\forall i,j\le t}$ is the covariance matrix and $\sigma^{2}_f$ is the output noise variance.
\begin{theorem}[$\varphi$-divergence Regret]\label{theorem_phi_regret}
Suppose the target function is bounded, meaning that $M = \sup_{(\bx,c) \in \mathcal{X} \times \mathcal{C}}\card{f(\bx,c)} < \infty$ and suppose $f$ has bounded RKHS norm with respect to $k$. For any lower semicontinuous convex $\varphi: \mathbb{R} \to (-\infty,\infty]$ with $\varphi(1) = 0$, if there exists a monotonic invertible function $\Gamma_{\varphi}: [0,\infty) \to \mathbb{R}$ such that $\operatorname{TV}(p,q) \leq \Gamma_{\varphi}(\mathsf{D}_{\varphi}(p,q))$, the following holds
\begin{align*}
    R_T(\varphi) \leq \frac{ \sqrt{8 T \beta_T \gamma_T}}{\log(1 + \sigma^{-2}_f)}  + \bigl(2M + \sqrt{\beta_T} \bigr) \sum_{t=1}^T \Gamma_{\varphi}( \varepsilon_t),
\end{align*}
with probability $1 - \delta$, where $\beta_t = 2 ||f||^2_k + 300 \gamma_t \ln^3(t/\delta)$,  $\gamma_t$ is the maximum information gain as defined above, and $\sigma_f$ is the standard deviation of the output noise.
\end{theorem}

%\textbf{Proof (Sketch)} 
%The proof first considers the total variation case ($\varphi(u) = \card{u - 1}$ and modifies the standard regret analysis which considers the regret at round $t$. In particular, the standard regret analysis uses the classical result from Theorem 6 in \cite{Srinivas_2010Gaussian} which links the target function to the Gaussian Process mean prediction and the variance with probability $1 - \delta$ and consequently relates this quantity to the maximum information gain. The main difference in this setting is one of the terms contains an infinum over the $\varphi$ divergence ball which by merit of Theorem \ref{theorem:main} reduces to a variance term which we bound by assumption. In order to generalize to the case for any $\varphi$ divergence, we utilize the existence of a function $\Gamma_{\varphi}$ which allows us to have 
%\begin{align*}
%&\braces{q \in \Delta(\mathcal{C}) : \mathsf{D}_{\varphi}(q,p_t) \leq \varepsilon_t }\\ &\subseteq \braces{q \in \Delta(\mathcal{C}) : \operatorname{TV}(q,p_t) \leq \Gamma_{\varphi}(\varepsilon_t) },
%\end{align*}
%which allows us to upper bound the regret from a $\varphi$-divergence ball with the total variation case and consequently completes the proof.

The full proof can be found in the Appendix Section \ref{supp-formal}. We first remark that with regularity assumptions on $f$, sublinear analytical bounds for $\gamma_T$ are known for a range of kernels, e.g., given $ \mathcal{X} \times \mathcal{C} \subset \realset^{d+1}$ we have for the RBF kernel, $\gamma_T \le \mathcal{O} \left( \log(T)^{d+2} \right)$ or for the Matérn kernel with $\nu >1$, $\gamma_T \le \mathcal{O} \left( T^{\frac{(d+1)(d+2)}{2\nu+(d+1)(d+2)}}(\log T) \right)$. The second term in the bound is directly a consequence of DRO and by selecting $\varepsilon_t = 0$, it will vanish since any such $\Gamma_{\varphi}$ will satisfy $\Gamma_{\varphi}(0) = 0$. To ensure sublinear regret, we can select $\varepsilon_t = \Gamma_{\varphi}^{-1}\bracket{\frac{1}{\sqrt{t} + \sqrt{t+1}}}$, noting that the second term will reduce to  $\sum_{t=1}^T\varepsilon_t \leq \sqrt{T}$. Finally, we remark that the existence of $\Gamma_{\varphi}$ is not so stringent since for a wide choices of $\varphi$, one can find inequalities between the Total Variation and $D_{\varphi}$, to which we refer the reader to \citet{sason2016f}. For the examples discussed above, we can select $\Gamma_{\varphi}(t) = t$ for the TV. For the $\chi^2$ and $\operatorname{KL}$ cases, one can choose $\Gamma_{\chi^2}(b) = 2\sqrt{\frac{b}{1+b}}$ and $\Gamma_{\operatorname{KL}}(b) = 1 - \exp(-b)$. 

%In particular, it is shown in Theorem~3.1 of \cite{csiszar1972class} that for any strictly convex $\varphi$ there exists a real-valued function $\Gamma_{\varphi}$ such that $\lim_{u \to 0^{+}}\Gamma_{\varphi}(u) = 0$ and $\operatorname{TV}(p,q) \leq \Gamma_{\varphi}(\mathsf{D}_{\varphi}(p,q))$. While we require monotonicity of $\Gamma_{\varphi}$, this result suggests that the existence of such functions is not unreasonable to assume.

%\vu{can we make some comments about our regret results. It is similar/tighter/looser/better... in some sense to the bound of \cite{kirschner2020distributionally}?}
%\hisham{The bounds here are much tighter as functions of the radius $\varepsilon$ however it is difficult to compare since $\varepsilon$ is a different parameter for both. For example $\operatorname{MMD}(p,q) \leq \varepsilon$ is a smaller set than $\operatorname{TV}(p,q) \leq \varepsilon$ for the same choice of $\varepsilon$. I will definitely include a qualitative discussion. }
%\vu{okay, no problem. We can show that our approach is superior: (i) simple + computed in closed-form (ii) without ad-hoc approximation (like MMD paper) (iii) empirically strong (iv) share similar regret bound.}

%\section{Discussion}

%\vu{this section is important to highlight the pros and cons of our approach. before showing the experiments, let  convince the readers intuitively why our approach is better and why they should accept our paper}

\section{Experiments \label{sec:experiments}}

\textbf{Experimental setting.} The experiments are repeated using $30$ independent runs. We set $|C|=30$ which should be sufficient to draw $c \stackrel{\textrm{iid}}{\sim} q$ in one-dimensional space to compute Eqs. (\ref{eq:alpha_x2},\ref{eq:alpha_tv}). We optimize the GP hyperparameter (e.g., learning rate) by maximizing the GP log marginal likelihood [\citenum{Rasmussen_2006gaussian}]. We will release the Python implementation code in the final version.%All methods are implemented in Python. 

\textbf{Baselines.} We consider the following baselines for comparisons.
\textit{Rand}: we randomly select $\bx_t$ irrespective of $c_t$.
\textit{BO}: we follow the GP-UCB [\citenum{Srinivas_2010Gaussian}] to perform standard Bayesian optimization (ignoring the context $c_t$). The selection at each iteration is $\bx_t = \argmax_{\bx} \mu(\bx) + \beta_t \sigma(\bx)$.
\textit{Stable-Opt}: we consider the worst-case robust optimization presented in \citet{bogunovic2018adversarially}. The selection at each iteration $\bx_t = \argmax_{\bx} \argmin_{c} \mu(\bx,c) + \beta_t \sigma(\bx,c)$.
\textit{DRBO MMD} [\citenum{kirschner2020distributionally}]:  Since there is no official implementation available, we have tried our best to re-implement the algorithm. %This approach relies on solving the inner adversary problem, which is a linear program with convex constraints.  The formulation and the theory continue to hold for large or continuous context sets, but finding a tractable algorithmic approximation is computational challenging.

% This program can be solved efficiently but is of context set size $|C|$, which currently limits the method to relatively small context sets.

%To avoid overflow when exponentiating large values, we have utilized the Log-mean-exp trick to implement the KL-divergence DRO.

We consider the popular benchmark functions\footnote{\hyperlink{https://www.sfu.ca/~ssurjano/optimization.html}{https://www.sfu.ca/~ssurjano/optimization.html} } with different dimensions $d$. To create a context variable $c$, we pick the last dimension of these functions to be the context input while the remaining $d-1$ dimension becomes the input $\bx$.

\subsection{Ablation Studies} \label{sec:ablation_study}
To gain understanding into how our framework works, we consider two popular settings below. %we illustrate the selection of $\bx_t$ over time in Fig. \ref{fig:ablation_study}. 

\textbf{DRBO solution is different from stochastic solution.} In Fig. \ref{fig:sto_dro_different}, the vanilla BO tends to converge greedily toward the stochastic solution (non-distributionally robust) $\argmax_{\bx} f(\bx,\cdot)$. Thus, BO keeps exploiting in the locality of $\argmax_{\bx} f(\bx,\cdot)$ from iteration $15$. On the other hand, all other DRBO methods will keep exploring to seek for the distributionally robust solutions. Using the high value of $\epsilon_t \in \{0.5,1\}$ will result in the best performance.

\textbf{DRBO solution is identical to stochastic solution.} When the stochastic and robust solutions coincide at the same input $\bx^*$, the solution of BO will be equivalent to the solution of DRBO methods. This is demonstrated by Fig. \ref{fig:sto_dro_same}. Both stochastic and robust approaches will quickly identify the optimal solution (see the $\bx_t$ selection). We learn empirically that setting $\epsilon_t \rightarrow 0$ will lead to the best performance. This is because the DRBO setting will become the standard BO.% and in other words we don't need DRO in this setting.

The best choice of $\epsilon$ depends on the property of the underlying function, e.g., the gap between the stochastic and DRBO solutions. In practice, we may not be able to identify these scenarios in advance. Therefore, we can use the adaptive value of $\epsilon_t$ presented in Section \ref{sec:convergence_analysis}. Using this adaptive setting, the performance is stable, as illustrated in the figures.

%\paragraph{Selection of input $\bx_t$ and observed output $f(\bx_t,c_t)$ over iterations.}
%To gain understanding into how our framework works, we illustrate the selection of $x_t$ and the output $y_t$ over time in Fig. \ref{fig:selection_xy_main}. It is interesting to see that the vanilla BO tends to converge early toward the stochastic solution (non-distributionally robust) $\argmax_{\bx} f(\bx,\cdot)$. Thus, BO keeps exploiting in the locality of $\argmax_{\bx} f(\bx,\cdot)$ from iteration $15$. On the other hand, all other DRO methods will keep exploring around to seek for the distributionally robust solutions. Our approaches (both $\chi^2$ and TV) converge into the DRO solution from iteration $40$ while DRO MMD fails to converge in this example. We refer to the appendix [] for further illustrations.

% \begin{wrapfigure}{R}{0.6\textwidth}
% \begin{centering}
% \includegraphics[width=0.25\textwidth]{content/figs/Rosenbrock_regret_different_eps.pdf}
% \includegraphics[width=0.25\textwidth]{content/figs/branin_regret_different_eps.pdf}
% \end{centering}
% \caption{Different robustness level $\epsilon$} \label{fig:different_eps}
% \end{wrapfigure}

\subsection{Computational efficiency} \label{sec_experiment_computational_comparison}
\begin{wrapfigure}{R}{0.47\textwidth}%\begin{centering}
\vspace{-28pt}
%\begin{wrapfigure}{0.5\textwidth}
%\begin{centering}
 %\centering
\includegraphics[width=0.46\textwidth]{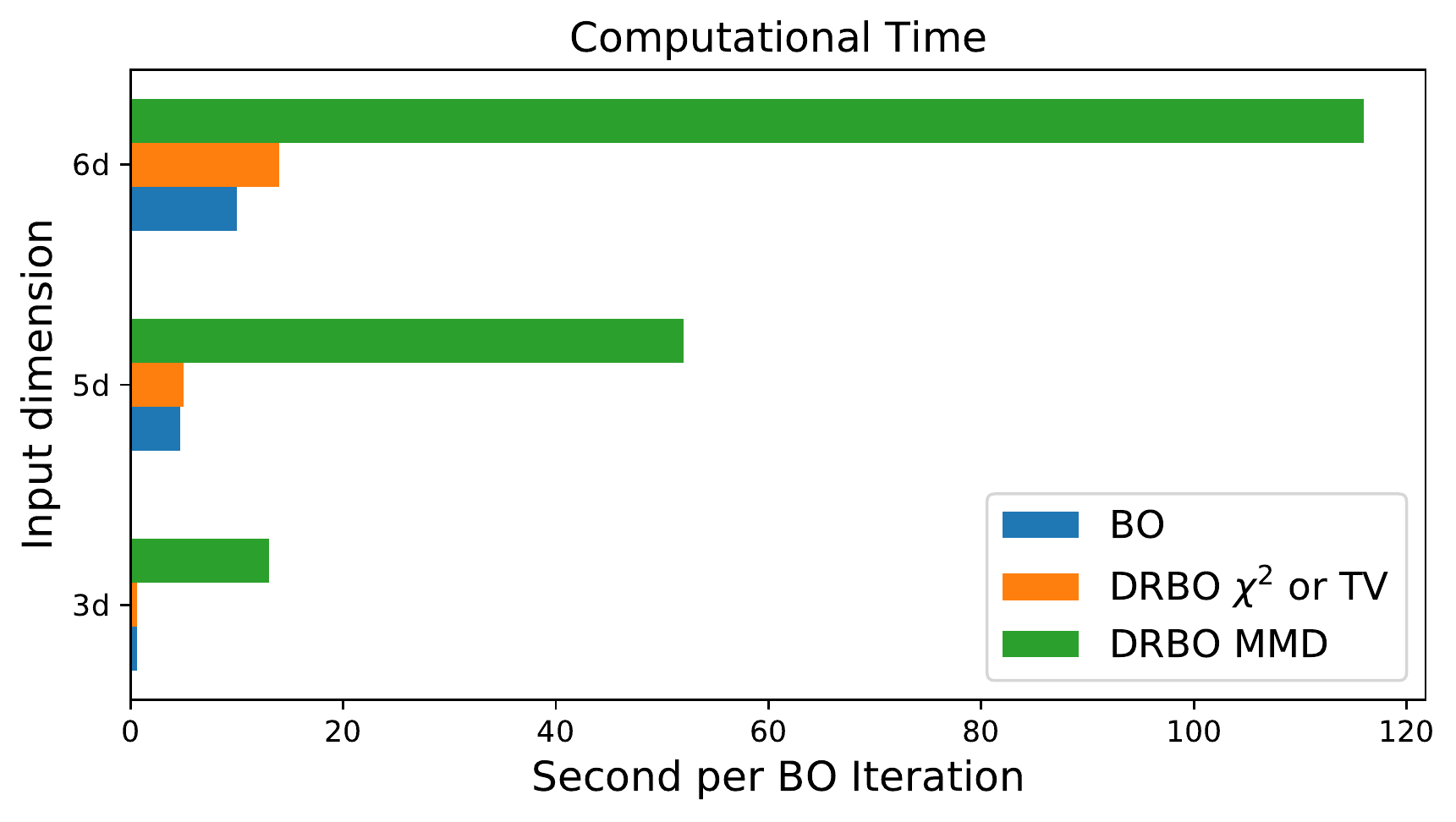}
%\end{centering}
\vspace{-5pt}
\caption{We compare the computational cost across methods. Our proposed DRBO using $\chi^2$ and TV take similar cost per iteration which is significantly lower than the DRBO MMD [\citenum{kirschner2020distributionally}]. } \label{fig:time_comparison}
\vspace{-0pt}
\end{wrapfigure}

\begin{figure*}
\vspace{-5pt}
\begin{centering}
\includegraphics[width=0.322\textwidth]{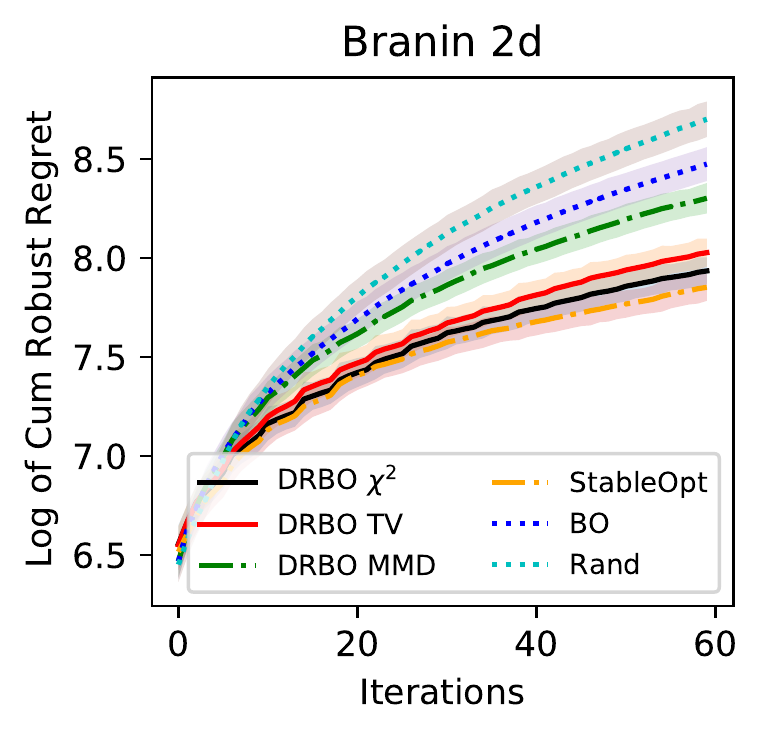}
\includegraphics[width=0.322\textwidth]{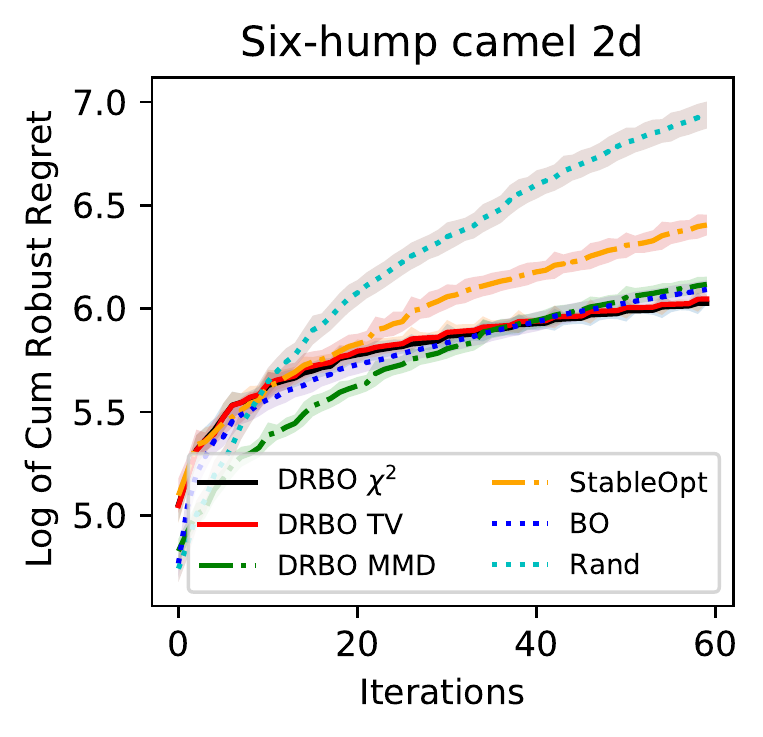}
%\vspace{-5pt}
%\includegraphics[width=0.33\textwidth]{content/figs/Alphine1_eps_0.100000_regret.pdf}
%\includegraphics[width=0.31\textwidth]{content/figs/Cosines_eps_0.100000_regret.pdf}
%\includegraphics[width=0.33\textwidth]{content/figs/Goldstein_eps_0.100000_regret.pdf}
\includegraphics[width=0.34\textwidth]{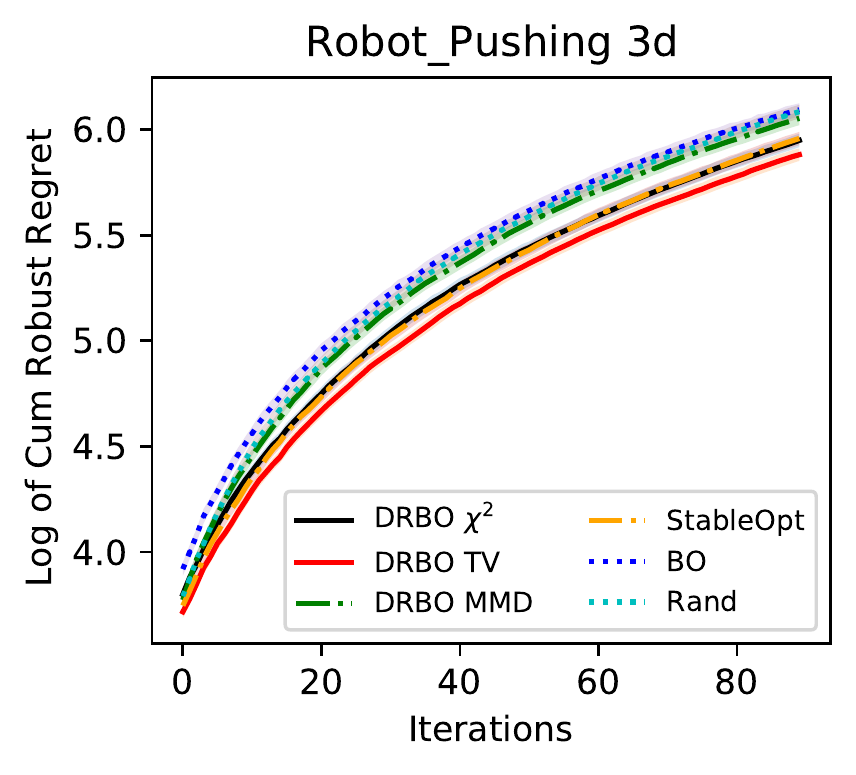}
\end{centering}
\vspace{-15pt}
\caption{Cumulative robust regret across algorithms. The results show that the proposed $\chi^2$ and TV achieve the best performance across benchmark functions. Random and vanilla BO approaches perform poorly which do not take into account the robustness criteria. Best viewed in color.} \label{fig:benchmark_functions}
\vspace{-15pt}
\end{figure*}

\begin{figure*}[h]
\vspace{-4pt}
\begin{centering}
\includegraphics[width=0.33\textwidth]{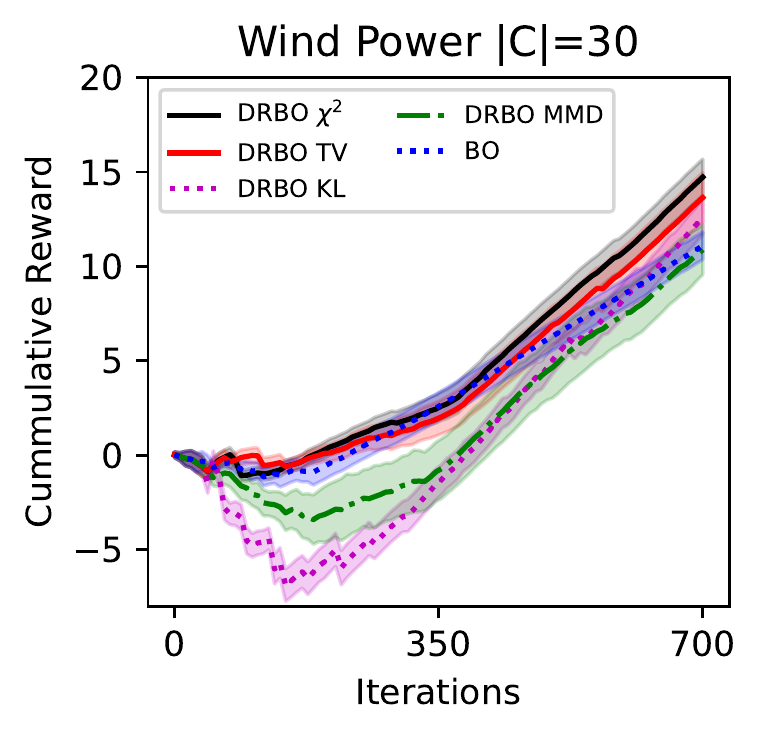}
\includegraphics[width=0.32\textwidth]{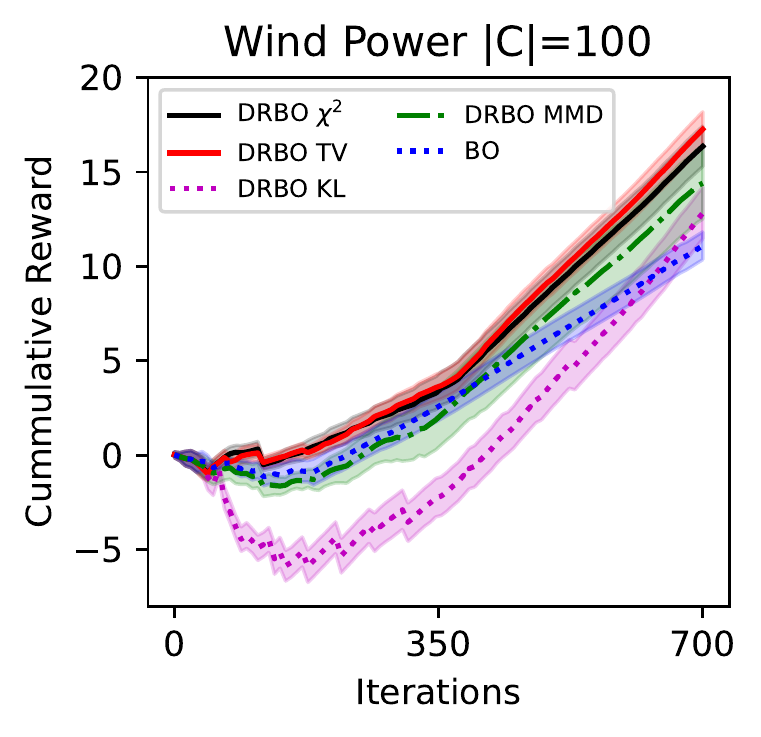}
\includegraphics[width=0.32\textwidth]{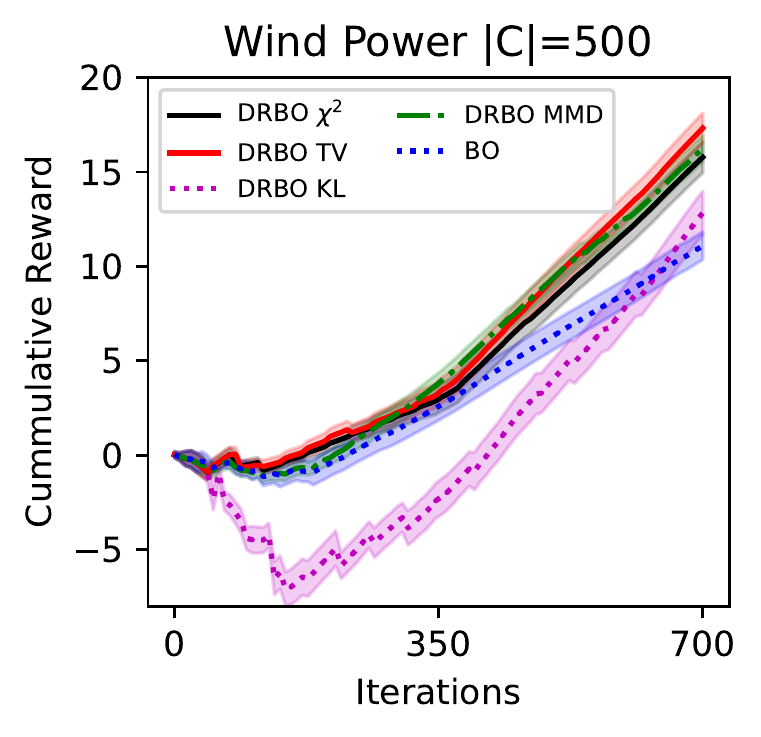}
\end{centering}
\vspace{-10pt}
\caption{All divergences improve with larger $|C|$. However, MMD comes with the quadratic cost.} \label{fig:windpower_wrt_C}% w.r.t. $|C|$.
\vspace{-6pt}
\end{figure*}

The key benefit of our framework is simplifying the existing intractable computation by providing the closed-form solution. Additional to improving the quality, we demonstrate this advantage in terms of computational complexity. Our main baseline for comparison is the MMD [\citenum{kirschner2020distributionally}]. As shown in Fig. \ref{fig:time_comparison}, our DRBO is consistently faster than the constraints linear programming approximation used for MMD. This gap is substantial in higher dimensions. In particular, as compared to \citet{kirschner2020distributionally}, our DRBO is $5$-times faster in \textit{5d} and $10$-times faster in \textit{6d}. %Our DRBO consumes similar time per iteration as the vanilla BO.

%\hspace*{-4cm}

%\paragraph{Different robustness level $\epsilon$.}
%We study the optimization performance with different level of robustness $\epsilon=\{0.5,0.1,0.001\}$ in Fig. \ref{fig:ablation_study} (\textit{right}). Using total variance divergence, we clearly show that the robustness performance increase with larger $\epsilon$.

\subsection{Optimization performance comparison}
We compare the algorithms in Fig. \ref{fig:benchmark_functions} using the robust (cumulative) regret defined in Eq. (\ref{eq:robust_regret}) which is commonly used in DRO literature [\citenum{kirschner2020distributionally,DRBQO}]. The random approach does not make any intelligent information in making decision, thus performs the worst. While BO performs better than random, it is still inferior comparing to other distributionally robust optimization approaches. The reason is that BO does not take into account the context information in making the decision. The StableOpt [\citenum{bogunovic2018adversarially}] performs relatively well that considers the worst scenarios in the subset of predefined context. This predefined subset can not cover all possible cases as opposed to the distributional robustness setting.

The MMD approach [\citenum{kirschner2020distributionally}] needs to solve the inner adversary  problem using  linear programming with convex constraints, additional to the main optimization step. As a result, the performance of MMD is not as strong as our TV and $\chi^2$. Our proposed approach does not suffer this pathology and thus scale well in continuous and high dimensional settings of context input $c$. %This program can be solved efficiently but is of size of context discretization $|C|$, which currently limits the method to relatively small context sets. 
\vspace{-1em}

\paragraph{Real-world functions.} %\vu{i am working on the real-world experiments in the next few days}
We consider the deterministic version of the robot pushing objective from \citet{Wang_2017Max}. The goal is to find a good pre-image for pushing an object to a target location. The 3-dimensional function takes as input the robot location $(r_x, r_y) \in [-5, 5]^2$ and pushing duration $r_t \in [1, 30] $. %The input domain is continuous: $r_x, r_y \in [5, 5]$ and $r_t \in [1, 30]$. %, and outputs $f(r_x, r_y, r_t)=5d_{end}$, where $d_{end}$ is the distance from the pushed object to the target location.
We follow \citet{bogunovic2018adversarially} to twist this problem in which there is uncertainty regarding the precise target location,
so one seeks a set of input parameters that is robust against a number of different potential pushing duration which is a context. %We present the comparison in Right Fig. \ref{fig:benchmark_functions}. Our proposed DRBO approaches outperform the baselines, especially the Total Variation (TV).

We perform an experiment on Wind Power dataset [\citenum{bogunovic2018adversarially}] and vary the context dimensions $|C| \in \{30, 100,500\}$ in Fig. \ref{fig:windpower_wrt_C}. When $|C|$ enlarges, our DRBO $\chi^2$, TV and KL improves. However, the performances do not improve further when increasing $|C|$ from $100$ to $500$. Similarly, MMD improves with $|C|$, but it comes with the quadratic cost w.r.t. $|C|$. Overall, our proposed DRBO still performs favourably in terms of optimization quality and computational cost
than the MMD.

%, with publicly available code \footenote{\hyperlink{https://github.com/zi-w/Max-value-Entropy-Search}}

% \begin{figure}
% \vspace{-7pt}
% \begin{centering}
% \includegraphics[width=0.4\textwidth]{content/figs/Robot_Pushing_eps_0.500000_regret.pdf}
% \end{centering}
% \vspace{-5pt}
% \caption{Robot Pushing} \label{fig:real_function}
% \vspace{-5pt}
% \end{figure}

\section{Conclusions, Limitations and Future works \label{sec:conclusion_limitation}}
\vspace{-1em}
%In this work, we targeted the DRBO formulation in \citet{kirschner2020distributionally} which was applied to the MMD distance which was associated with various limitations. We showed how one can study this formulation with respect to $\varphi$-divergences and derived a new algorithm that removes much of the computational burden, along with a sublinear regret bound. We compared the performance of our method against others, and showed that our results unveil a deeper connection between regularization and robustness, which serves useful conceptually.
In this work, we showed how one can study the DRBO formulation with respect to $\varphi$-divergences and derived a new algorithm that removes much of the computational burden, along with a sublinear regret bound. We compared the performance of our method against others, and showed that our results unveil a deeper connection between regularization and robustness, which serves useful conceptually.
\vspace{-5pt}
\paragraph{Limitations and Future Works} One of the limitations of our framework is in the choice of $\varphi$, for which we provide no guidance. For different applications, different choices of $\varphi$ would prove to be more useful, the study of which we leave for future work.
%\vu{3-4 sentences for conclusion and what do you think will be the future work?}
%\hisham{I added three notes for future work here:}
%\hisham{ We showed how one can utilize variance regularization and its connection to $\varphi$-divergences however one can study this phenomena more broadly like Integral Probability Metric (IPM) DRO where its been shown that other forms of regularization guarantees distributional robustness in the setting of supervised learning.}
%\hisham{There have been works studying how DRO is linked to adversarial ML and so we lay the foundation to further study a link between DRO-BO and adversarial BO.}
%\hisham{Develop the results for more specialized settings of the contexts $\mathcal{C}$. Perhaps by making certain assumptions of what kind of contexts, we could adapt a strategy for selecting $p_t(c)$ using the results here.}
% Thanks, the points look great!
\section*{Acknowledgements}
We would like to thank the anonymous reviewers for providing feedback.

\bibliographystyle{plainnat}
%\bibliography{biblio}

%\bibliographystyle{unsrtnat}

\bibliography{vunguyen, references}
%\bibliographystyle{icml2023}

%%%%%%%%%%%%%%%%%%%%%%%%%%%%%%%%%%%%%%%%%%%%%%%%%%%%%%%%%%%%%%%%%%%%%%%%%%%%%%%
%%%%%%%%%%%%%%%%%%%%%%%%%%%%%%%%%%%%%%%%%%%%%%%%%%%%%%%%%%%%%%%%%%%%%%%%%%%%%%%
% APPENDIX
%%%%%%%%%%%%%%%%%%%%%%%%%%%%%%%%%%%%%%%%%%%%%%%%%%%%%%%%%%%%%%%%%%%%%%%%%%%%%%%
%%%%%%%%%%%%%%%%%%%%%%%%%%%%%%%%%%%%%%%%%%%%%%%%%%%%%%%%%%%%%%%%%%%%%%%%%%%%%%%
\newpage
%\appendix
\onecolumn

\clearpage

%\onecolumn
\onecolumn

\section{Gaussian Process Modeling with Input $x$ and Context $c$}

\paragraph{Gaussian Processes.} We follow a popular choice in BO [\citenum{Shahriari_2016Taking}] to use GP as a surrogate model for optimizing  $f$. A GP [\citenum{Rasmussen_2006gaussian}] defines a probability distribution
over functions $f$ under the assumption that any subset of points
$\left\{ (\bx_{i},f(\bx_{i})\right\} $ is normally distributed. Formally,
this is denoted as:
\begin{align*}
f(\bx)\sim \text{GP}\left(m(\bx),k(\bx,\bx')\right),
\end{align*}
where $m\left(\bx\right)$ and $k\left(\bx,\bx'\right)$ are the mean
and covariance functions, given by $m(\bx)=\mathbb{E}\left[f(\bx)\right]$
and $k(\bx,\bx')=\mathbb{E}\left[(f(\bx)-m(\bx))(f(\bx')-m(\bx'))^{T}\right]$. For predicting $f_{*}=f\left(\bx_{*}\right)$ at a new data point $\bx_{*}$,  
the conditional probability
follows a univariate Gaussian distribution as $p\bigl(f_{*}\mid \bx_*, [\bx_1...\bx_N],[y_1,...y_N] \bigr)\sim\mathcal{N}\left(\mu\left(\bx_{*}\right),\sigma^{2}\left(\bx_{*}\right)\right)$.
Its mean and variance are given by:
\begin{minipage}{0.49\textwidth}
\begin{align}
\mu\left(\bx_{*}\right)= & \mathbf{k}_{*,N}\mathbf{K}_{N,N}^{-1}\mathbf{y}, \label{eq:gp_mean}
\end{align}
\end{minipage}
\begin{minipage}{0.49\textwidth}
\begin{align}
\sigma^{2}\left(\bx_{*}\right)= & k_{**}-\mathbf{k}_{*,N}\mathbf{K}_{N,N}^{-1}\mathbf{k}_{*,N}^{T}\label{eq:gp_var}
\end{align}
\end{minipage}

where $k_{**}=k\left(\bx_{*},\bx_{*}\right)$, $\bk_{*,N}=[k\left(\bx_{*},\bx_{i}\right)]_{\forall i\le N}$ and $\bK_{N,N}=\left[k\left(\bx_{i},\bx_{j}\right)\right]_{\forall i,j\le N}$.
 As GPs give full uncertainty information with any prediction, they
provide a flexible nonparametric prior for Bayesian optimization.
We refer to \citet{Rasmussen_2006gaussian}
for further details on GPs.

\paragraph{Gaussian Process with Input $x$ and Context $c$.}
%  is Section \ref{sec:preliminary} 
We have presented the definition of Gaussian process where the input includes a variable $\bx \in \realset^d $. 
Given the additional context variable $c \in \realset$, it is natural to consider a GP model in which the input is the concatenation of $[\bx, c] \in \realset^{d+1}$. In particular, we can write a GP [\citenum{Rasmussen_2006gaussian}] as:
\begin{align*}
f( [\bx,c] )\sim \text{GP}\left(m([\bx,c]),k([\bx,c],[\bx',c'])\right),
\end{align*}
where $m\left([\bx,c]\right)$ and $k\left([\bx,c],[\bx',c']\right)$ are the mean
and covariance functions. For predicting $f_{*}=f\left([\bx_{*}, c_*]\right)$ at a new data point $[\bx_{*},c_*]$,  
the conditional probability
follows a univariate Gaussian distribution as $p\bigl(f_{*}\mid [\bx_*, c_*],...  \bigr)\sim\mathcal{N}\left(\mu\left( [\bx_{*},c_*] \right),\sigma^{2}\left([\bx_{*},c_*]\right)\right)$. %\bigl[ [\bx_1, c_1]...[\bx_N,c_N] \bigr],[y_1,...y_N] 
The mean and variance are given by:

\begin{minipage}{0.49\textwidth}
\begin{align}
\mu\left([\bx_{*},c_*]\right)= & \mathbf{k}_{*,N}\mathbf{K}_{N,N}^{-1}\mathbf{y}, \label{eq:gp_mean_xc}
\end{align}
\end{minipage}
\begin{minipage}{0.49\textwidth}
\begin{align}
\sigma^{2}\left([\bx_{*},c_*]\right)= & k_{**}-\mathbf{k}_{*,N}\mathbf{K}_{N,N}^{-1}\mathbf{k}_{*,N}^{T}\label{eq:gp_var_xc}
\end{align}
\end{minipage}

where $k_{**}=k\left([\bx_{*},c_*],[\bx_{*},c_*]\right)$, $\bk_{*,N}=[k\left([\bx_{*},c_*],[\bx_{i},c_i]\right)]_{\forall i\le N}$ and $\bK_{N,N}=\left[k\left([\bx_{i},c_i],[\bx_{j},c_j]\right)\right]_{\forall i,j\le N}$.

\section{Proofs of Main Results}\label{supp-formal}
In the sequel, when we say a function is measurable, we attribute it with respect to the Borel $\sigma$-algebras based on the Polish topologies. We remark that a similar proof to ours has appeared in \citet{ahmadi2012entropic}, which is not specific to BO objective yet also we require compactness of the set $\mathcal{C}$. We only require compactness of $\mathcal{C}$ to get compactness of $\Delta(\mathcal{C})$ however similar arguments can be made when $\mathcal{C}$ is not compact using the vague topology such as in \citep{liu2018inductive, husain2019primal, husain2022adversarial}.
\begin{theorem} \textcolor{blue}{(Theorem \ref{theorem:main} in the main paper)}
Let $\phi: \mathbb{R} \to (-\infty, \infty]$ be a convex lower semicontinuous mapping such that $\phi(1) = 0$. For any $\epsilon > 0$, it holds that
\begin{align*}
    &\sup_{\bx \in \mathcal{X}} \inf_{q \in B_{\phi}^{t}(p)} \E_{q(c)}[f(\bx,c)] =\sup_{\bx \in \mathcal{X},  \lambda \geq 0, b \in \mathbb{R}} \bracket{ b - \lambda \epsilon_t - \lambda \E_{p_{t}(c)}\left[ \varphi^{\star} \bracket{\frac{b - f(\bx,c)}{\lambda} }\right]  }.
\end{align*}
\end{theorem}
\begin{proof}
For a fixed $\bx \in \mathcal{X}$, we first introduce a Lagrangian variable $\lambda \geq 0$ that acts to enforce the ball constraint  $\mathsf{D}_{\varphi}(p,q) \leq \varepsilon$:
\begin{align}
    \inf_{q \in B_{\phi}^{t}(p)} \E_{q(c)}[f(\bx,c)] &= \inf_{q \in \Delta(\mathcal{C})} \sup_{\lambda \geq 0} \bracket{ \E_{q(c)}[f(\bx,c)] - \lambda \bracket{\varepsilon_t - \mathsf{D}_{\varphi}(q,p_t)}}\\
    &\stackrel{(1)}{=}  \sup_{\lambda \geq 0} \inf_{q \in \Delta(\mathcal{C})} \bracket{ \E_{q(c)}[f(\bx,c)] - \lambda \bracket{\varepsilon_t - \mathsf{D}_{\varphi}(q,p_t)}}\\
    &= \sup_{\lambda \geq 0} \bracket{ -\lambda \varepsilon_t - \sup_{q \in \Delta(\mathcal{C})} \bracket{\E_{q(c)}[-f(\bx,c)] - \lambda \mathsf{D}_{\varphi}(q,p_t) } }\\
    &\stackrel{(2)}{=} \sup_{\lambda \geq 0} \bracket{-\lambda \varepsilon_t - \inf_{b \in \mathbb{R}}\bracket{\lambda \E_{p_{t}(c)}\left[\varphi^{\star} \bracket{\frac{b - f(\bx,c)}{\lambda}}\right] - b } }\\
    &= \sup_{ \lambda \geq 0, b \in \mathbb{R}} \bracket{ b - \lambda \epsilon_t - \lambda \E_{p_{t}(c)}\left[ \varphi^{\star} \bracket{\frac{b - f(\bx,c)}{\lambda} }\right]  }
\end{align}
$(1)$ is due to Fan's minimax Theorem \citep[Theorem~2]{fan1953minimax} noting that for any $\bx \in \mathcal{X}$, the mapping $q \mapsto \E_{q(c)}[f(\bx,c)]$ is linear and for any $\varphi$ chosen as stated in the Theorem, the mapping $q \mapsto \mathsf{D}_{\varphi}(q,p_t)$ is convex and lower semi-continuous. Furthermore noting that $\mathcal{C}$ is compact, we also have that $\Delta(\mathcal{C})$ is compact [\citenum{villani2008optimal}]. $(2)$ is due to a standard result due to the (restricted) Fenchel dual of the $\varphi$-divergence, see Eq. (22) of \citet{liu2018inductive} for example. We state the result in the following lemma, which depends on the convex conjugate $\varphi^{\star}(u) = \sup_{u'} \bracket{u u' - \varphi(u')}$.
\begin{lemma}
For any measurable function $h: \mathcal{C} \to \mathbb{R}$ and convex lower semi-continuous function $\varphi: \mathbb{R} \to (-\infty,\infty]$ with $\varphi(1) = 0$, it holds that
\begin{align*}
    &\sup_{q \in \Delta(\mathcal{C})} \bracket{\E_{q(c)}[h(c)] - \lambda \mathsf{D}_{\varphi}(q,p) } = \inf_{b \in \mathbb{R}}\bracket{\lambda \E_{p_{t}(c)}\left[ \varphi^{\star}\bracket{\frac{b + h(c)}{\lambda}} \right] - b},
\end{align*}
where $p \in \Delta(\mathcal{C})$ and $\lambda > 0$.
\end{lemma}
%\begin{proof}
%First note that $\mathsf{D}_{\varphi}(q,p) < +\infty \iff q \ll p$, which by the Radon-Nikonym derivative, means there exists a measurable function $r: \mathcal{C} \to \mathbb{R}$ such that $r = dq/dp$, notably with $\E_{p_{t}(c)}[r(c)] = 1$. Therefore, we have
%\begin{align*}
%    \sup_{q \in \Delta(\mathcal{C})} \bracket{\E_{q(c)}[h(c)] - \lambda \mathsf{D}_{\varphi}(q,p) } &= \sup_{q \in \Delta(\mathcal{C}) : q \ll p} \bracket{\E_{q(c)}[h(c)] - \lambda \mathsf{D}_{\varphi}(q,p) }\\
%    &= \sup_{r: \mathcal{C} \to \mathbb{R} : \E_{p_{t}(c)}[r(c)] = 1 } \bracket{\E_{q(c)}[h(c)] - \lambda \mathsf{D}_{\varphi}(q,p) }\\
%    &= \sup_{r: \mathcal{C} \to \mathbb{R} : \E_{p_{t}(c)}[r(c)] = 1 } \bracket{\E_{p_{t}(c)}[r(c) h(c)] - \lambda \E_{p_{t}(c)}[\varphi(r(c))] }\\
%    &= \sup_{r: \mathcal{C} \to \mathbb{R} } \inf_{b \in \mathbb{R}} \bracket{\E_{p_{t}(c)}[r(c) h(c)- \lambda\varphi(r(c))] - b (\E_{p_{t}(c)}[r(c)] - 1) }
%\end{align*}
%\end{proof}
\end{proof}
For each of the specific derivations below, we recall the standard derivations for $\varphi^{\star}$, which can be found in many works, for example in \cite{nowozin2016f}.

\begin{example}[$\chi^2$-divergence] \textcolor{blue}{(Example \ref{example_chi2} in the main paper)}
If $\phi(u) = (u-1)^2$, then we have
\begin{align*}
    &\sup_{\bx \in \mathcal{X}} \inf_{q \in B_{\phi}^{t}(p)} \E_{c \sim q}[f(\bx,c)] = \sup_{\bx \in \mathcal{X}} \bracket{ \E_{p_{t}(c)}[f(\bx,c)] - \sqrt{\varepsilon_t\operatorname{Var}_{p_{t}(c)}[f(\bx,c)] }}.
\end{align*}
\end{example}
\begin{proof}
In this case we have $(\lambda \varphi)^{\star}(u) = \frac{u^2}{4\lambda } + u$ and so we have
\begin{align*}
    \sup_{b \in \mathbb{R}} \bracket{b -  \E_{p_{t}(c)}\left[ \frac{(b - f(\bx,c))^2}{4\lambda} + b - f(\bx,c)  \right] } &= \sup_{b \in \mathbb{R}} \bracket{ \E_{p_{t}(c)}\left[ f(\bx,c)\right] - \E_{p_{t}(c)}\left[ \left(b - f(\bx,c) \right)^2   \right] }\\
    &= \E_{p_{t}(c)}[f(\bx,c)] - \frac{1}{4\lambda} \inf_{b \in \mathbb{R}} \E_{p_{t}(c)}\left[ \left(b - f(\bx,c) \right)^2 \right]\\
    &= \E_{p_{t}(c)}[f(\bx,c)] - \frac{1}{4\lambda} \operatorname{Var}_{p_{t}(c)}[f(\bx,c)].
\end{align*}
Combining this with the original objective and by Theorem \ref{theorem:main} yields
\begin{align*}
\inf_{q \in B_{\phi}^{t}(p)} \E_{c \sim q}[f(\bx,c)] &=    \sup_{  \lambda \geq 0, b \in \mathbb{R}} \bracket{ b - \lambda \epsilon_t - \lambda \E_{p_{t}(c)}\left[ \varphi^{\star} \bracket{\frac{b - f(\bx,c)}{\lambda} }\right]  }\\ &=  \E_{p_{t}(c)}[f(\bx,c)] \nonumber  
- \inf_{\lambda \geq 0} \bracket{\lambda \varepsilon_t + \frac{1}{4\lambda}\operatorname{Var}_{p_{t}(c)}[f(\bx,c)]  }\\
 &=      \E_{p_{t}(c)}[f(\bx,c)] \nonumber 
 - \sqrt{\varepsilon_t \operatorname{Var}_{p_{t}(c)}[f(\bx,c)] }.
\end{align*}
where the last equation is using the arithmetic and geometric means inequality.
\end{proof}

\begin{example}[Total Variation]\textcolor{blue}{(Example \ref{example_tv} in the main paper)}
\label{mainthm:tvd}
If $\varphi(u) = \card{u-1}$, then we have
\begin{align*}
    &\sup_{\bx \in \mathcal{X}} \inf_{q \in B_{\phi}^{t}(p)} \E_{c \sim q}[f(\bx,c)] =  \sup_{\bx \in \mathcal{X}} \bracket{ \E_{p(c)}[f(\bx,c)] - \frac{\varepsilon_t}{2}\bracket{\sup_{c \in \mathcal{C}}f(\bx,c) - \inf_{c \in \mathcal{C}}f(\bx,c)  } }.
\end{align*}
\end{example}
\begin{proof}
In this case, the conjugate of $\varphi$ is $(\lambda \varphi)^{\star}(u) = 0$ if $\card{u} \leq \lambda$ and $+\infty$ otherwise. Therefore, when considering the right-hand side of Theorem \ref{theorem:main}, we will require $\lambda$ to be larger than $\card{b - f(\bx,c)}$ for all $c \in \mathcal{C}$ for the expression to be finite. In this case, for a fixed $b \in \mathbb{R}$, the optimization over $\lambda \geq 0$ becomes
\begin{align}
    &\sup_{b \in \mathbb{R}} \sup_{\lambda \geq 0} \bracket{b - \lambda \varepsilon_t -  \lambda \E_{p_{t}(c)}\left[ \varphi^{\star} \bracket{\frac{b - f(\bx,c)}{\lambda}} \right] } \\&=  \sup_{b \in \mathbb{R}}  \bracket{ \E_{p_{t}(c)}[f(\bx,c)] -  \sup_{c \in \mathcal{C}} \card{b - f(\bx,c)}\varepsilon_t }\\
    &= \E_{p_{t}(c)}[f(\bx,c)] - \inf_{b \in \mathbb{R}} \sup_{c \in \mathcal{C}} \card{b - f(\bx,c)} \varepsilon_t.
\end{align}
We need the following Lemma to proceed.
\begin{lemma}
For any set of numbers $A \subset \mathbb{R}$, let $\overline{a}, \underline{a} \in A$ be the maximum and minimal elements. It then holds that
\begin{align*}
    \inf_{b \in \mathbb{R}} \sup_{a \in A} \card{b - a} = \frac{1}{2} \card{\overline{a} - \underline{a}}.
\end{align*}
\end{lemma}
\begin{proof}
First note that for any $b \in \mathbb{R}$, we have that $\sup_{a \in A} \card{b - a} = \max\bracket{\card{\overline{a} - b}, \card{\underline{a} - b}}$. The outer inf can be solved by setting $b = \frac{\overline{a} + \underline{a}}{2}$.
\end{proof}
The proof then concludes by noting setting $A = \braces{f(\bx,c) : c \in \mathcal{C}}$.
\end{proof}

\begin{lemma}[Theorem 6 from \citet{Srinivas_2010Gaussian}] \label{lemma_bound_mu_sigma_f}
Let $\delta \in (0,1)$. Assume the noise variable 
$\epsilon_t$ is uniformly bounded by $\sigma_f$. Define $\beta_t = 2 ||f||^2_k + 300 \gamma_t \ln^3(t/\delta)$. Then,
\begin{align}
    P \Bigl(\forall T, \forall \bx, \bigl| \mu_T(\bx) - f(\bx) \bigr| \le \beta_T^{\frac{1}{2}} \sigma_T(\bx) \Bigr) \geq 1- \delta.
\end{align}
\end{lemma}

We first prove a regret bound for the total variation which will be instrumental in proving the regret bound for any general $\varphi$, given the existence of $\Gamma_{\varphi}$.

\begin{theorem}[Total Variation Regret] \label{thm_TV_Regret_appendix}
Let $M = \sup_{(\bx,c) \in \mathcal{X} \times \mathcal{C}}\card{f(\bx,c)}$ and suppose $f$ lives in an RKHS. If $\varphi(u) = \card{u-1}$, it then holds that
\begin{align*}
    R_T(\varphi) \leq \frac{ \sqrt{8 T \beta_T \gamma_T}}{\log(1 + \sigma_f^{-2})}  + \bigl(2M + \sqrt{\beta_T} \bigr) \sum_{t=1}^T \varepsilon_t,
\end{align*}
where $\sigma_f$ is the standard deviation of the output noise.
\end{theorem}
\begin{proof}
We first define the GP predictive mean and variance as
\begin{align*}
    \mu_t(\bx,c) &= k_t([\bx,c])^{\intercal}\bracket{ \bK_t + \idenmat_t \sigma^2_f}^{-1}y_t\\
    \sigma_t(\bx,c)^2 &= k([\bx,c],[\bx,c])- k_t([\bx,c])^{\intercal}\bracket{ \bK_t + \idenmat_t \sigma^2_f}^{-1}k_t([\bx,c]),
\end{align*}
which are defined in Eqs. (\ref{eq:gp_mean_xc},\ref{eq:gp_var_xc}) where $\bK_t=\left[k\left([\bx_{i},c_i],[\bx_{j},c_j]\right)\right]_{\forall i,j\le t}$. The proof begins by first bounding the regret at time $t$ using a standard argument with a slight modification that uses our results.
\begin{align}
    r_t=&\inf_{q : \mathsf{D}_{\varphi}(q,p_t) \leq \varepsilon_t} \E_{q(c)}[f(\bx^{*},c)] - \inf_{q : \mathsf{D}_{\varphi}(q,p_t) \leq \varepsilon_t} \E_{q(c)}[f(\bx_t,c)] \nonumber \\
    &\stackrel{(1)}{\leq} \E_{p_t(c)}[f(\bx^{*},c) - \mu(\bx^{*},c)] + \E_{p_t(c)}[\mu(\bx^{*},c)]  - \inf_{q : \mathsf{D}_{\varphi}(q,p_t) \leq \varepsilon_t} \E_{q(c)}[f(\bx_t,c)] \nonumber\\
    &\stackrel{(2)}{\leq} \sqrt{\beta_t}\E_{p_t(c)}[\sigma_t(\bx^*,c)]] +  \E_{p_t(c)}[\mu(\bx^{*},c)]  - \inf_{q : \mathsf{D}_{\varphi}(q,p_t) \leq \varepsilon_t} \E_{q(c)}[f(\bx_t,c)] \nonumber\\
    &= \sqrt{\beta_t}\E_{p_t(c)}[\sigma_t(\bx^*,c)] + \E_{p_t(c)}[\mu(\bx^{*},c)]  - \inf_{q : \mathsf{D}_{\varphi}(q,p_t) \leq \varepsilon_t} \E_{q(c)}[f(\bx_t,c)] \nonumber\\
    &\stackrel{(3)}{=} \sqrt{\beta_t}\E_{p_t(c)}[\sigma_t(\bx^*,c)] + \E_{p_t(c)}[\mu(\bx^{*},c)]  - \E_{p_t(c)}[f(\bx_t,c)] + \frac{\varepsilon_t}{2} \bracket{\sup_{c \in \mathcal{C}} f(\bx_t,c) - \inf_{c \in \mathcal{C}} f(\bx_t,c) } \nonumber\\
    &= \sqrt{\beta_t}\E_{p_t(c)}[\sigma_t(\bx^*,c)]  + \E_{p_t(c)}[\mu(\bx^{*},c) - f(\bx_t,c)] + \frac{\varepsilon_t}{2} \bracket{\sup_{c \in \mathcal{C}} f(\bx_t,c) - \inf_{c \in \mathcal{C}} f(\bx_t,c) } \nonumber \\
    &\stackrel{(4)}{\leq} \sqrt{\beta_t}\E_{p_t(c)}[\sigma_t(\bx_t,c)]  + \E_{p_t(c)}[\mu(\bx_t,c) - f(\bx_t,c)] + \frac{\varepsilon_t}{2} \bracket{\sup_{c \in \mathcal{C}} f(\bx_t,c) - \inf_{c \in \mathcal{C}} f(\bx_t,c) } \nonumber \\
    & \qquad + \frac{\epsilon_t}{2} \E_{p_t(c)} \bigl[ \max \mu_t(\bx^*,c) - \min \mu_t(\bx^*,c)  \bigr] \nonumber \\
    &\stackrel{(5)}{\leq} 2\sqrt{\beta_t}\E_{p_t(c)}[\sigma_t(\bx_t,c)]  + \varepsilon_t \bigl(2M + \sqrt{\beta_t} \bigr), \label{eq:r_t}
\end{align}
where $(1)$ holds due to selecting $p_t$ and introducing $\E_{q(c)}[\mu(\bx^{*},c)]$, $(2)$ is due to the result that $\card{f(\bx,c) - \mu(\bx,c)} \leq \sqrt{ \beta_t } \sigma_t(\bx,c)$ for all $x, c \in \mathcal{X} \times \mathcal{C}$ with probability at least $1 - \delta$ as stated in Lemma \ref{lemma_bound_mu_sigma_f}. $(3)$ is due to Theorem \ref{mainthm:tvd}. Step $(4)$ is due to the choice of $\bx_t$ since it satisfies:
\begin{align*}
\E_{p_t(c)} \Bigl[ \mu(\bx_t,c) + \sqrt{\beta_t}\sigma_t(\bx_t,c) -\frac{\epsilon_t}{2} \big( \max \mu_t(\bx_t,c) - \min \mu_t(\bx_t,c) \big) \Bigl]\geq \\ 
\E_{p_t(c)} \Bigl[\mu(\bx,c) + \sqrt{\beta_t}\sigma_t(\bx,c) -\frac{\epsilon_t}{2} \big( \max \mu_t(\bx,c) - \min \mu_t(\bx,c) \big) \Bigr],
\end{align*}
for all $\bx \in \mathcal{X}$ and therefore 
\begin{align*}
\E_{p_t(c)} \Bigl[\mu(\bx_t,c) + \sqrt{\beta_t}\sigma_t(\bx_t,c) -\frac{\epsilon_t}{2} \big( \max \mu_t(\bx_t,c) - \min \mu_t(\bx_t,c) \big) \Bigr] \geq \\ 
\E_{p_t(c)} \Bigl[ \mu(\bx^*,c) + \sqrt{\beta_t}\sigma_t(\bx^*,c) -\frac{\epsilon_t}{2} \big( \max \mu_t(\bx^*,c) - \min \mu_t(\bx^*,c) \big) \Bigr].
\end{align*}
Thus, we have step (4) as
\begin{align*}
    \E_{p_t(c)}[\mu(\bx^{*},c)]  + \sqrt{\beta_t}\E_{p_t(c)}[\sigma_t(\bx^{*},c)]     \leq & \,\, \E_{p_t(c)}[\mu(\bx_t,c)] + \sqrt{\beta_t}\E_{p_t(c)}[\sigma_t(\bx_t,c)] \\
    & + \frac{\epsilon_t}{2} \E_{p_t(c)} \bigl[ \max \mu_t(\bx^*,c) - \min \mu_t(\bx^*,c)  \bigr].
\end{align*}
Finally, $(5)$ holds due to another application of $\card{f(\bx,c) - \mu(\bx,c)} \leq \beta_t \sigma_t(\bx,c)$ for all $[\bx, c] \in [\mathcal{X} \times \mathcal{C}]$ from Lemma \ref{lemma_bound_mu_sigma_f}, $\E_{p_t(c)} \bigl[ \sup_{c \in \mathcal{C}} f(\bx_t,c) - \inf_{c \in \mathcal{C}} f(\bx_t,c) \bigr] \le 2M$, and $\E_{p_t(c)} \bigl[ \max \mu_t(\bx^*,c) - \min \mu_t(\bx^*,c)  \bigr] \le 2 [M + \sqrt{\beta_t} \sigma_t(\bx^*,c) ] \le 2 [M+\sqrt{\beta_t}]$ where $M = \sup_{(\bx,c) \in \mathcal{X} \times \mathcal{C}}\card{f(\bx,c)}$ and $\sigma_t(\bx,c)\le 1,\forall [\bx,c]$. For the final step of our proof, we follow the literature in Bayesian optimization [\citenum{Srinivas_2010Gaussian}] to introduce a sample complexity parameter, namely the maximum information gain:
\begin{align*}
\gamma_T := \max_{ \{(\bx_t,c_t) \}_{t=1}^T} \log \det (\idenmat_t + \sigma^{-2}_f \bK_T).
\end{align*}
where we use $\bK_T$ to denote the generated kernel matrix from dataset $D_T = \braces{[\bx_i,c_i]}_{i=1}^T  \subset [\mathcal{X} \times \mathcal{C}]$ at iteration $T$.
%In particular, we use $D_t = \braces{(\bx_i,c_i)}_{i=1}^t$ to denote the dataset at round $t$.
The information gain is used in the regret bounds for most of Bayesian optimization research [\citenum{Nguyen_ACML2017Regret,kirschner2020distributionally}].
\begin{lemma} (adapted Lemma 7 in \citet{Nguyen_ACML2017Regret})
The sum of the predictive variances is bounded by the maximum information gain $\gamma_T$.
That is $\forall \bx,c \in \mathcal{X \times C}, \sum_{t=1}^T \sigma_{t-1}^2 (\bx,c) \le \frac{2}{ \log(1+\sigma^{-2}_f) } \gamma_T$ where $\sigma_f$ is the standard deviation of the output noise.
\end{lemma}

Using the above Lemma of maximum information gain, we take the square of the term $2\sqrt{\beta_t}\E_{p_t(c)}[\sigma_t(\bx_t,c)]$ in Eq. (\ref{eq:r_t}) to have:
\begin{align}
    \sum_{t=1}^T 4 \beta_t \E_{p_t(c)}[\sigma^2_t(\bx_t,c)] \le \frac{8 \beta_T \gamma_T}{\log(1 + \sigma_f^{-2})} 
\end{align}
By using Cauchy-Schwarz inequality, we get
\begin{align} \label{eq_sum_rt_to_MIG}
    \sum_{t=1}^T 2\sqrt{\beta_t}\E_{p_t(c)}[\sigma_t(\bx,c)] \le \sqrt{ \frac{8 T \beta_T \gamma_T}{\log(1 + \sigma_f^{-2})}} 
\end{align}
in which the term $T$ has been included in the right.

Using the above results we get
\begin{align}
 R_T(\varphi) &\leq 2 \sum_{t=1}^T \sqrt{\beta_t} \E_{p_t(c)}[\sigma_t(\bx,c)] + \bigl(2M + \sqrt{\beta_T} \bigr) \sum_{t=1}^T \varepsilon_t\\&\leq 2 \sqrt{\beta_T} \sum_{t=1}^T \max_c [\sigma_t(\bx,c)] + \bigl(2M + \sqrt{\beta_T} \bigr) \sum_{t=1}^T \varepsilon_t\\
 &\leq 2 \sqrt{\beta_T} \max_{q \in \braces{p_1,\ldots,p_T}}\sum_{t=1}^T  \E_{q(c)}[\sigma_t(\bx,c)] + \bigl(2M + \sqrt{\beta_T} \bigr) \sum_{t=1}^T \varepsilon_t\\
 &\leq 2 \sqrt{\beta_T} \max_{q \in \braces{p_1,\ldots,p_T}}\E_{q(c)}\left[\sum_{t=1}^T  \sigma_t(\bx,c)\right] + \bigl(2M + \sqrt{\beta_T} \bigr) \sum_{t=1}^T \varepsilon_t\\
 &\leq \frac{ \sqrt{8 T \beta_T \gamma_T} }{\log(1 + \sigma_f^{-2})}  + \bigl(2M + \sqrt{\beta_T} \bigr) \sum_{t=1}^T \varepsilon_t \label{eq_final_regret}
\end{align}
where we have used  Eq. (\ref{eq_sum_rt_to_MIG}) to obtained Eq. (\ref{eq_final_regret}).
\end{proof}

\begin{lemma}\label{lemma:telescoping}
For any $T > 0$, it holds that
\begin{align*}
    \sum_{t=1}^T \bracket{\frac{1}{\sqrt{t} + \sqrt{t+1}}} = \sqrt{T+1} - 1 \leq \sqrt{T},
\end{align*}
\end{lemma}
\begin{proof}
By simply rationalizing the denominator, we have that
\begin{align*}
  \frac{1}{\sqrt{t} + \sqrt{t+1}} = \sqrt{t+1} - \sqrt{t},
\end{align*}
and via a simple telescoping sum, the required result holds. A simple inequality will then yield the final inequality.
\end{proof}

\begin{theorem}[$\varphi$-divergence Regret]\textcolor{blue}{(Theorem \ref{theorem_phi_regret} in the main paper)}
Let $M = \sup_{(\bx,c) \in \mathcal{X} \times \mathcal{C}}\card{f(\bx,c)} < \infty$ and suppose $f$ has bounded RKHS norm with respect to $k$. For any lower semicontinuous convex $\varphi: \mathbb{R} \to (-\infty,\infty]$ with $\varphi(1) = 0$, if there exists a monotonic invertible function $\Gamma_{\varphi}: [0,\infty) \to \mathbb{R}$ such that $\operatorname{TV}(p,q) \leq \Gamma_{\varphi}(\mathsf{D}_{\varphi}(p,q))$, the following holds
\begin{align*}
    R_T(\varphi) \leq \frac{ \sqrt{8 T \beta_T \gamma_T}}{\log(1 + \sigma^{-2}_f)}  + \bigl(2M + \sqrt{\beta_T} \bigr) \sum_{t=1}^T \Gamma_{\varphi}( \varepsilon_t),
\end{align*}
with probability $1 - \delta$, where $\beta_t = 2 ||f||^2_k + 300 \gamma_t \ln^3(t/\delta)$,  $\gamma_t$ is the maximum information gain as defined above, and $\sigma_f$ is the standard deviation of the output noise.
\end{theorem}

\begin{proof}
The key aspect of this proof is to note that
\begin{align}
    \braces{q \in \Delta(\mathcal{C}) : \mathsf{D}_{\varphi}(q,p_t) \leq \varepsilon_t } \subseteq \braces{q \in \Delta(\mathcal{C}) : \operatorname{TV}(q,p_t) \leq \Gamma_{\varphi}(\varepsilon_t) },
\end{align}
which is due to the inequality and monotonicity of the $\Gamma_{\varphi}$ function. By following similar steps to the Total Variation derivation, we have
\begin{align}
    &\inf_{q : \mathsf{D}_{\varphi}(q,p_t) \leq \varepsilon_t} \E_{q(c)}[f(\bx^{*},c)] - \inf_{q : \mathsf{D}_{\varphi}(q,p_t) \leq \varepsilon_t} \E_{q(c)}[f(\bx_t,c)]\\
    &\leq \sqrt{\beta_t}\E_{p_t(c)}[\sigma_t(\bx^*,c)^2 + \mu(\bx^{*},c)]  - \inf_{q : \mathsf{D}_{\varphi}(q,p_t) \leq \varepsilon_t} \E_{q(c)}[f(\bx_t,c)]\\
    &\leq \sqrt{\beta_t}\E_{p_t(c)}[\sigma_t(\bx^*,c)^2 + \mu(\bx^{*},c)]  - \inf_{q  : \operatorname{TV}(q,p_t) \leq \Gamma_{\varphi}(\varepsilon_t)} \E_{q(c)}[f(\bx_t,c)].
\end{align}
By following the same decomposition as in the proof for the Total Variation, presented in Theorem \ref{thm_TV_Regret_appendix}, we achieve the same result except with the radius epsilon replaced with $\Gamma_{\varphi}(\varepsilon_t)$.
\end{proof}

\section{Extension to KL divergence \label{app_sec_kl_divergence}}

Our theoretical result for $\phi$-divergence can also be readily extended to handle KL divergence.
However, since the empirical result with KL divergence is inferior to the result with TV and $\chi$-divergences.

\begin{example}[KL-divergence] \label{example_kl}
Let $\phi(u) = u \log u$, we have
\begin{align*}
    \sup_{\bx \in \mathcal{X}} \inf_{q \in B_{\phi}^{t}(p)} \E_{c \sim q}[f(\bx,c)] = \nonumber \sup_{\bx \in \mathcal{X},  \lambda \geq 0} \bracket{ -\lambda \epsilon_t - \lambda \log \E_{p_{t}(c)} \left[ \exp\bracket{\frac{-f(\bx,c)}{\lambda}} \right]  }.
\end{align*}
\end{example}

% \begin{example}[KL-divergence]
% If $\varphi(u) = u \log u$, then we have
% \begin{align*}
%     &\sup_{\bx \in \mathcal{X}} \inf_{q \in B_{\phi}^{t}(p)} \E_{c \sim q}[f(\bx,c)] = \sup_{\bx \in \mathcal{X},  \lambda \geq 0} \bracket{ -\lambda \epsilon_t - \lambda \log \E_{p_{t}(c)} \left[ \exp\bracket{\frac{-f(\bx,c)}{\lambda}} \right]  }.
% \end{align*}
% \end{example}
\begin{proof}
Using Theorem \ref{theorem:main}, we derive $\varphi^{\star}$ for this choice which can easily be verified to be $\varphi^{\star}(t) = \exp(t-1)$. For simplicity, let $\lambda = 1$ and note that we have
\begin{align}
&\sup_{ b \in \mathbb{R}} \bracket{ b - \E_{p_{t}(c)}\left[ \exp \bracket{b - f(\bx,c) - 1 }\right]  }= \sup_{ b \in \mathbb{R}} \bracket{ b - \exp\bracket{b} \cdot A  },
\end{align}
where $A =\E_{p_{t}(c)}\left[ \exp \bracket{- f(\bx,c) - 1 }\right]$ is a constant. The above can easily be solved as it admits a differentiable one-dimensional objective and we get the largest value when $b = -\log A$ which then yields
\begin{align}
    &\sup_{ b \in \mathbb{R}} \bracket{ b - \E_{p_{t}(c)}\left[ \exp \bracket{b - f(\bx,c) - 1 }\right]  } = \log \E_{p_{t}(c)}\left[\exp(-f(\bx,c)) \right].
\end{align}
The proof concludes noting that $(\lambda \varphi)^{\star}(u) = \lambda \varphi^{\star}(u / \lambda)$ for any $\lambda > 0$.
\end{proof}
The Kullback-Leibler (KL) divergence [\citenum{kullback_1951_information}] is a popular choice when quantifying information shift, due to its link with entropy. There exists work that studied distributional shifts with respect to the KL divergence for label shifts [\citenum{zhang2020coping}]. Compared to the general theorem, the KL-divergence variant allows us to find $b \in \mathbb{R}$ in closed form. We remark that we place the KL divergence derivation here for its information-theoretic importance however in our experiments, we find that other choices of $\varphi$-divergence outperform the KL-divergence. We now show such examples and particularly illustrate that we can even solve for $\lambda \geq 0$ in closed form for these cases, yielding only a single maximization over $\mathcal{X}$. 

The regret bound for the case with KL divergence is presented in Theorem \ref{theorem_phi_regret} using $\Gamma_{\textrm{KL}}(t) = 1 - \exp(-t)$.

\begin{figure*}
\begin{centering}
\includegraphics[width=0.325\textwidth]{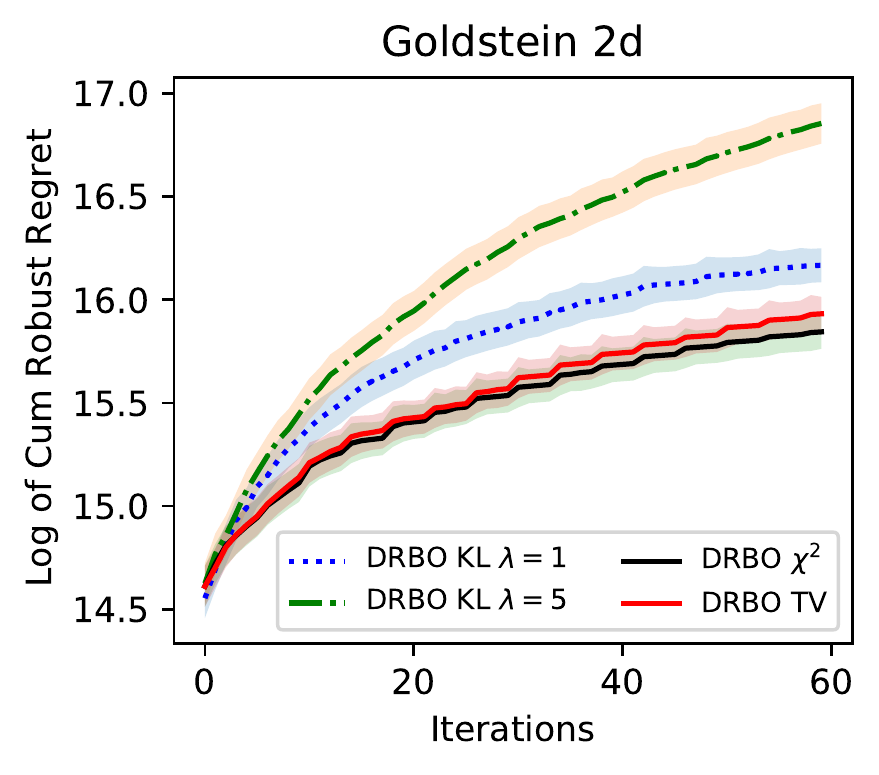}
\includegraphics[width=0.325\textwidth]{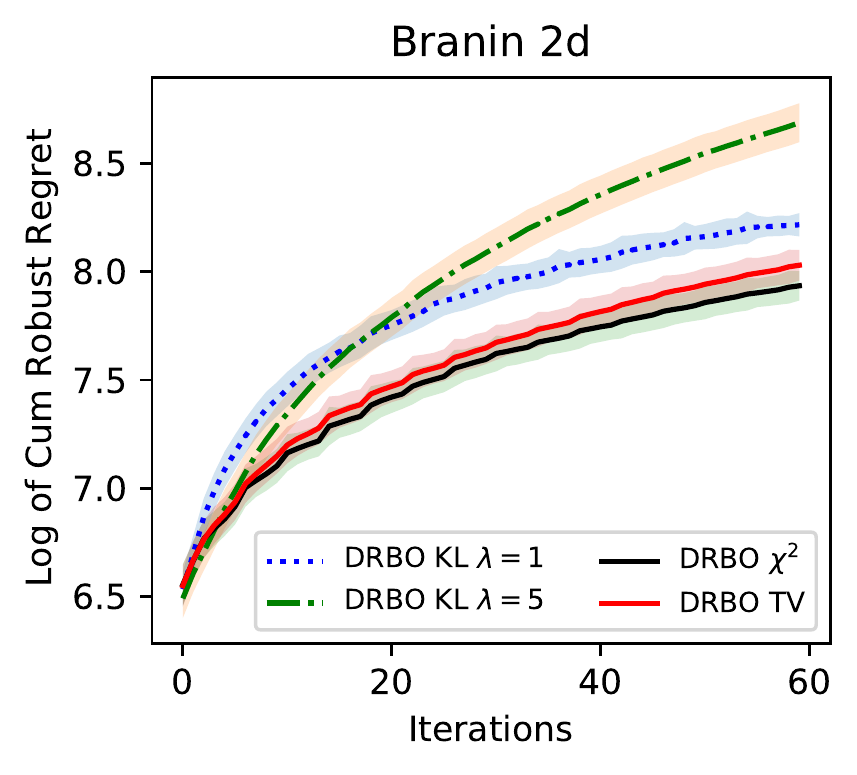}
\includegraphics[width=0.325\textwidth]{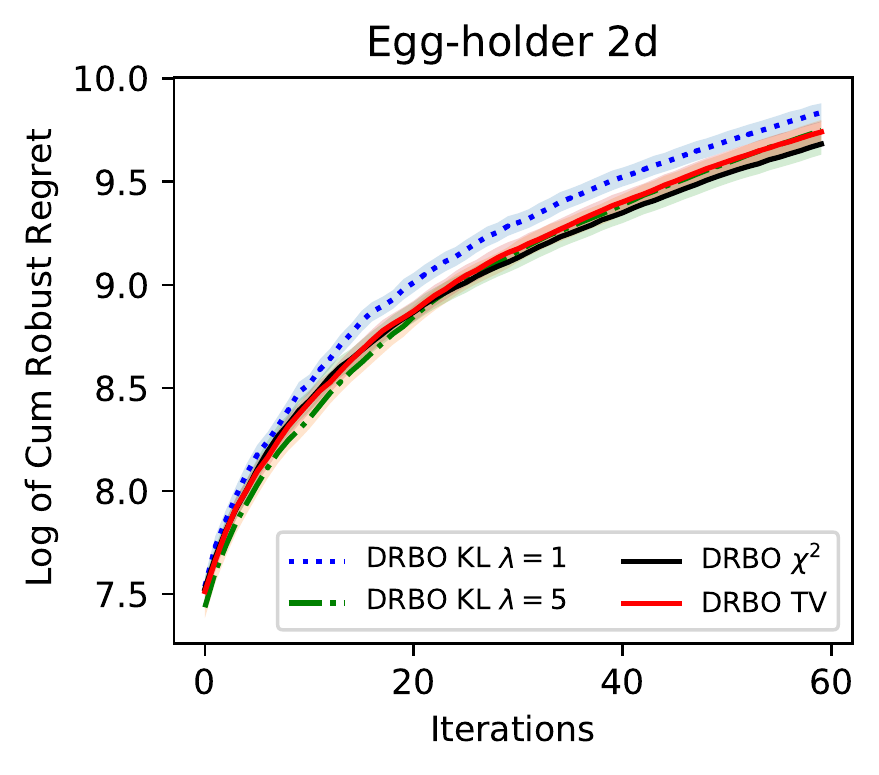}

\end{centering}

\caption{We empirically found that using TV and $\chi^2$ divergences typically obtain better performance than KL divergence. This can be due to the sensitivity of the additional hyperparameter $\lambda$ in KL. We have considered using $\lambda=1$ and $\lambda=5$ in these experiments. The best $\lambda$ is unknown in advance and depends on the functions.} \label{fig:kl_tv_chi2}
\end{figure*}

\subsection{Optimization with the GP Surrogate for KL-divergence}
To handle the distributional robustness, we have rewritten the objective function using $\phi$ divergences in Theorem \ref{theorem:main} with the KL in \textbf{Example} \ref{example_kl}.

%To sequentially select a next point $\bx_t$ for querying a black-box function. Given the observed context $c_t \sim q$ coming from the environment, we evaluate the black-box function as $y_t = f(\bx_t, c_t) + \eta_t$.

Similar to the cases of TV and $\chi^2$, we model the GP surrogate model using the observed data $\{\bx_i,y_i\}_{i=1}^{t-1}$ at each iteration $t$-th, . Then, we select a next point $\bx_t$ to query a black-box function by maximizing the acquisition function which is build on top of the GP surrogate: 
\begin{align*}
\bx_t = \arg \max_{\bx \in \mathcal{X} } \alpha(\bx).
\end{align*}
While our method is not restricted to the form of the acquisition function, for convenient in the theoretical analysis, we follow the GP-UCB [\citenum{Srinivas_2010Gaussian}]. Given the GP predictive mean and variance from Eqs. (\ref{eq:gp_mean_xc},\ref{eq:gp_var_xc}), we have the acquisition function for the KL in Example \ref{example_kl} as follows:
\begin{align} \label{eq:alpha_kl}
    \alpha^{KL}(\bx) &:=   -\lambda \epsilon_t - \lambda \log \E_{p_{t}(c)} \left[ \exp\bracket{\frac{-\frac{1}{|C|} \sum_{c} \left[ \mu_t(\bx,c) + \sqrt{\beta_t}\sigma_t(\bx,c) \right]}{\lambda}} \right] 
\end{align}
where $c \sim q$ can be  generated in a one dimensional space to approximate the expectation and the variance. In the experiment, we define $q$ as the uniform distribution to draw $c \sim q$.

\paragraph{Comparison of KL, TV and $\chi^2$ divergences.}
In Fig. \ref{fig:kl_tv_chi2}, we present the additional experiments showing the comparison of using different divergences including KL, TV and $\chi^2$. We show empirically that the KL divergence performs generally inferior than the TV and $\chi^2$.

%\section{Additional Experiments \label{app_additional_experiments}}
%\input{content/experiments}

\paragraph{Adaptive value of $\epsilon_t$.}
We show the adaptive value of $\epsilon_t$ by iterations for TV, $\chi^2$ and KL in Fig. \ref{fig:adaptive_epsilon}.

\begin{figure*}
\begin{centering}
\includegraphics[width=0.57\textwidth]{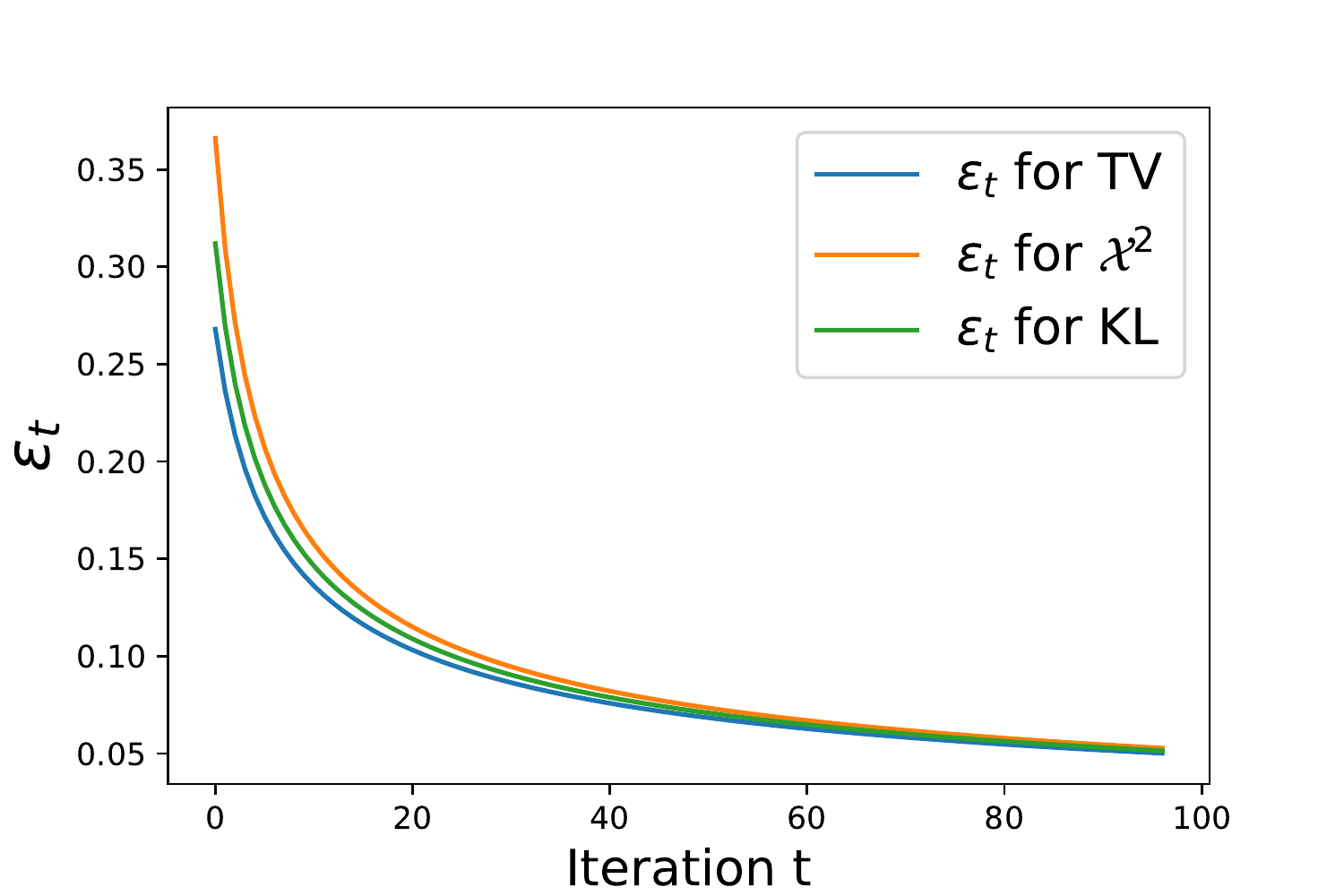}

\end{centering}

\caption{The adaptive value of $\epsilon_t$ over iterations, $ \lim_{t \rightarrow \infty}, \epsilon_t =0$. } \label{fig:adaptive_epsilon}
\end{figure*}

In Fig. \ref{fig:ablation_study_appendix}, we present additional visualization to complement the analysis in Fig. \ref{fig:ablation_study}. In particular, we illustrate three other functions including branin, goldstein and six-hump camel. The additional results are consistent with the finding presented in the main paper (Section \ref{sec:ablation_study} and Fig. \ref{fig:ablation_study}).

%Finally, we study the performance with different choices of epsilon.

%We visualize the benchmark functions in Fig. \ref{fig:viz_benchmark_func}  highlighting the difference between finding distributionally robust (DR) optimum and stochastic optimum.

%We show the fact that using the existing approaches for seeking stochastic solution will result in different, but wrong, solution as opposed to the distributionally robust approaches, such as the proposed method. For example in \textit{branin} function, the stochastic solution is located in bottom right corner (red star) while the DR solution is in the middle left (yellow dash line).

\begin{figure*}
%\vspace{-8pt}
\begin{subfigure}[b]{1\textwidth}
\includegraphics[width=0.33\textwidth]{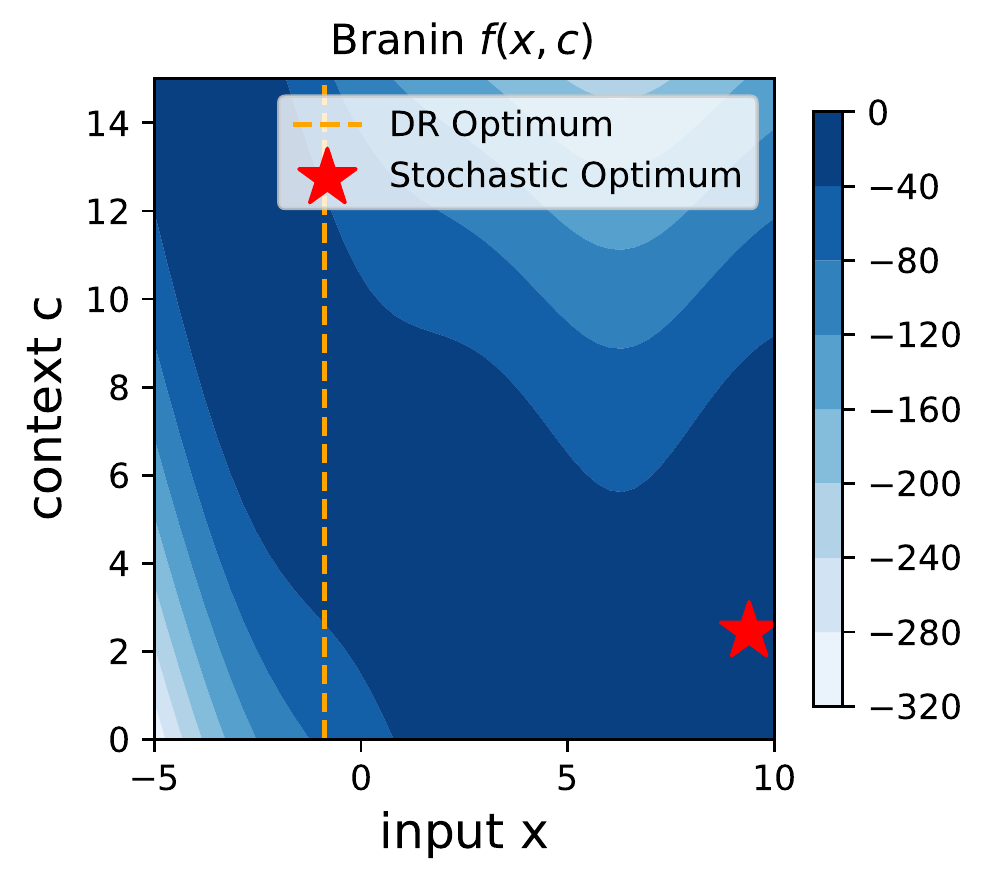}
\includegraphics[width=0.31\textwidth]{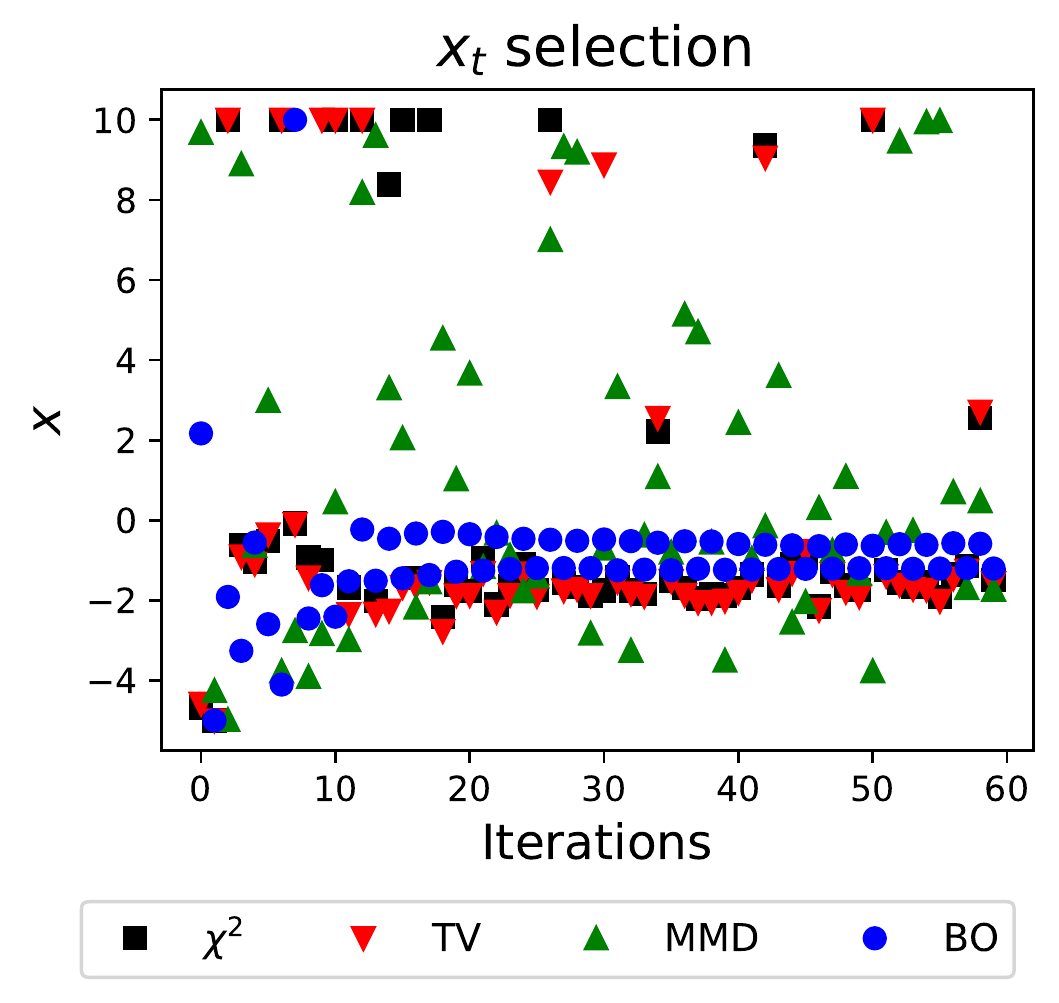}
\includegraphics[width=0.35\textwidth]{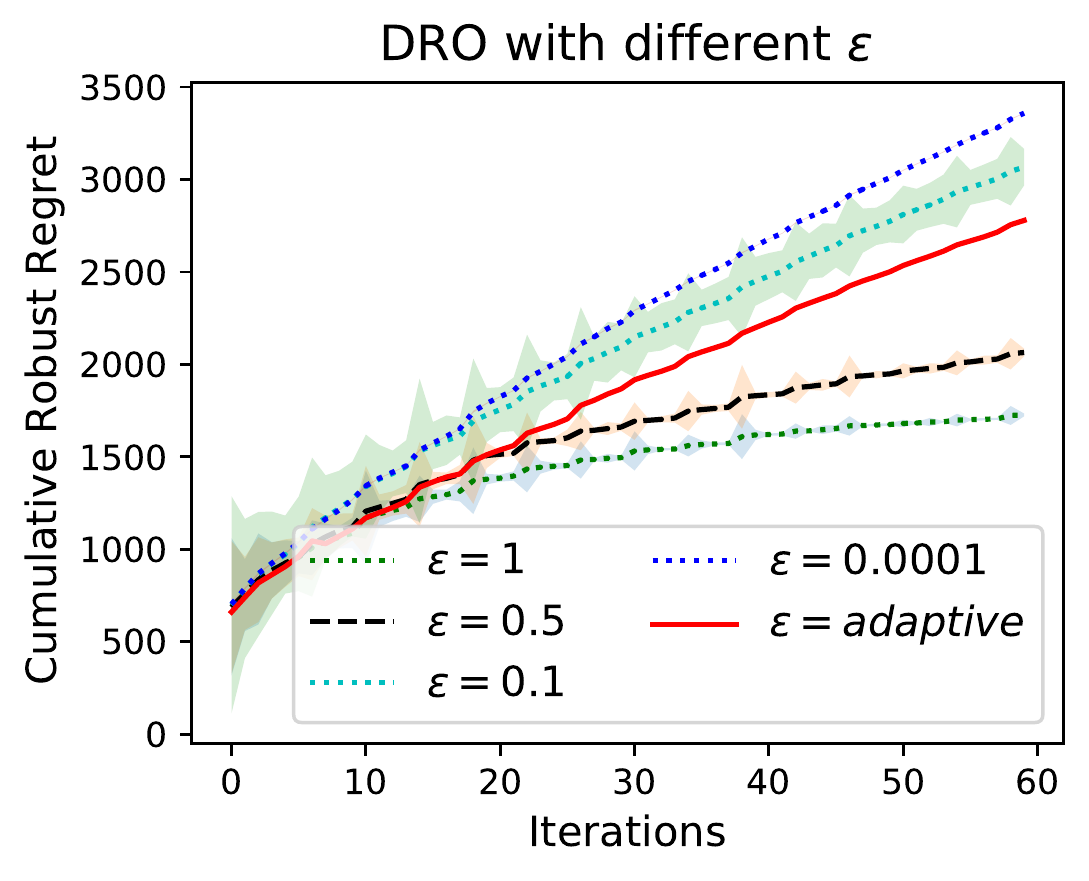}
%\vspace{-2pt}
\caption{Stochastic and DRO solution are different. The choices of $\epsilon=\{0.5,1\}$ result in the best performance.} \label{fig:appendix_sto_dro_different}
\end{subfigure}

\begin{subfigure}[b]{1\textwidth}
\includegraphics[width=0.34\textwidth]{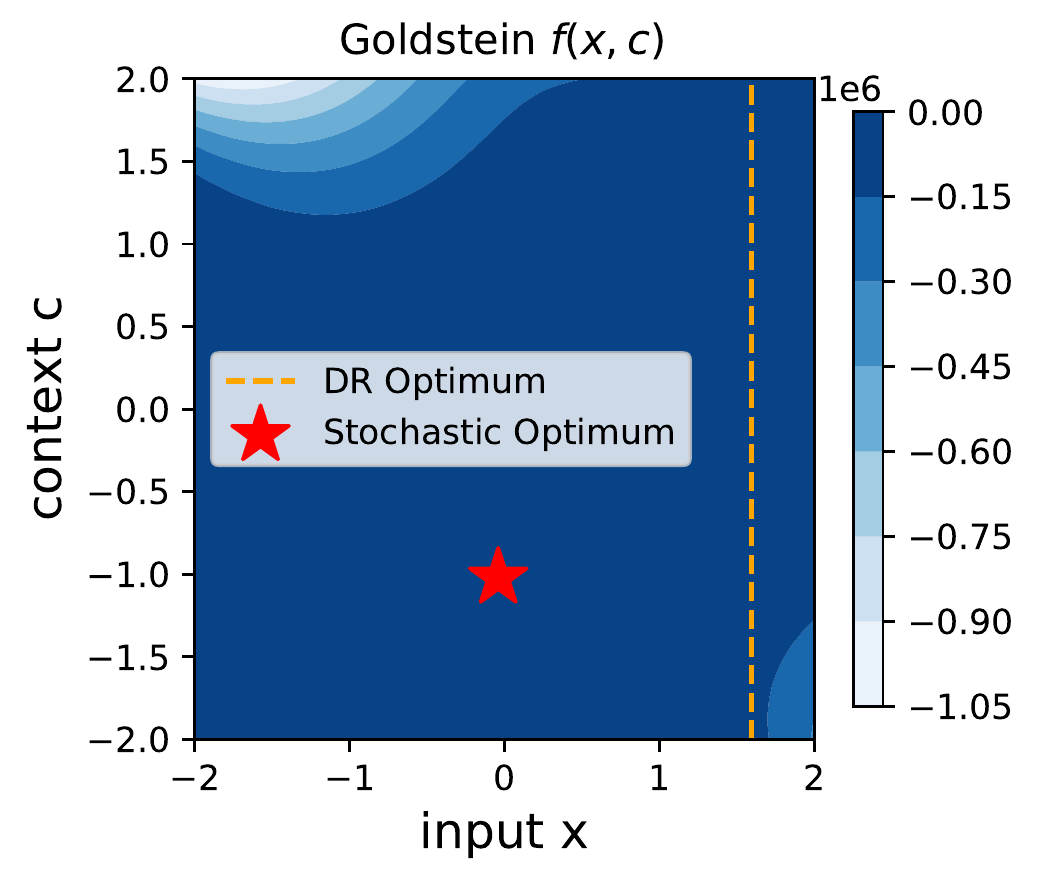}
\includegraphics[width=0.31\textwidth]{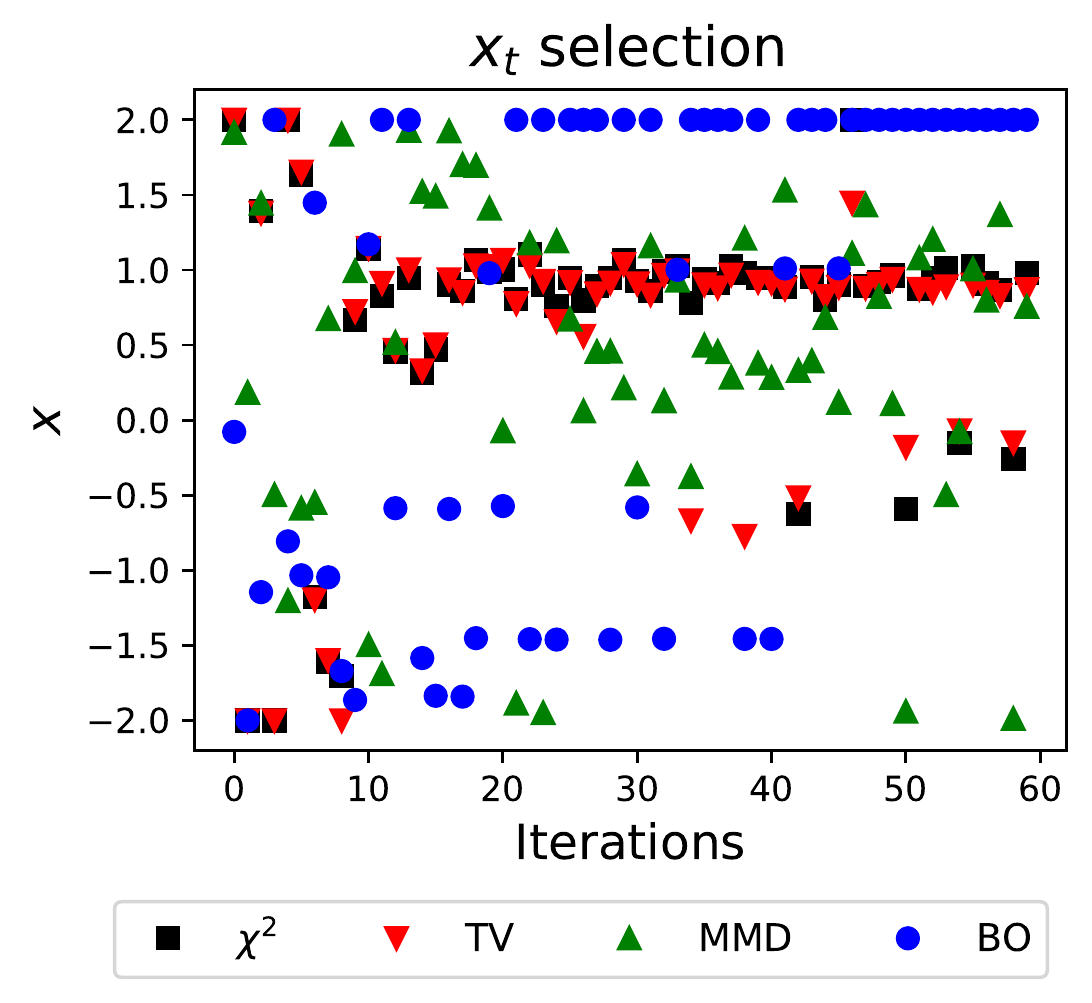}
\includegraphics[width=0.34\textwidth]{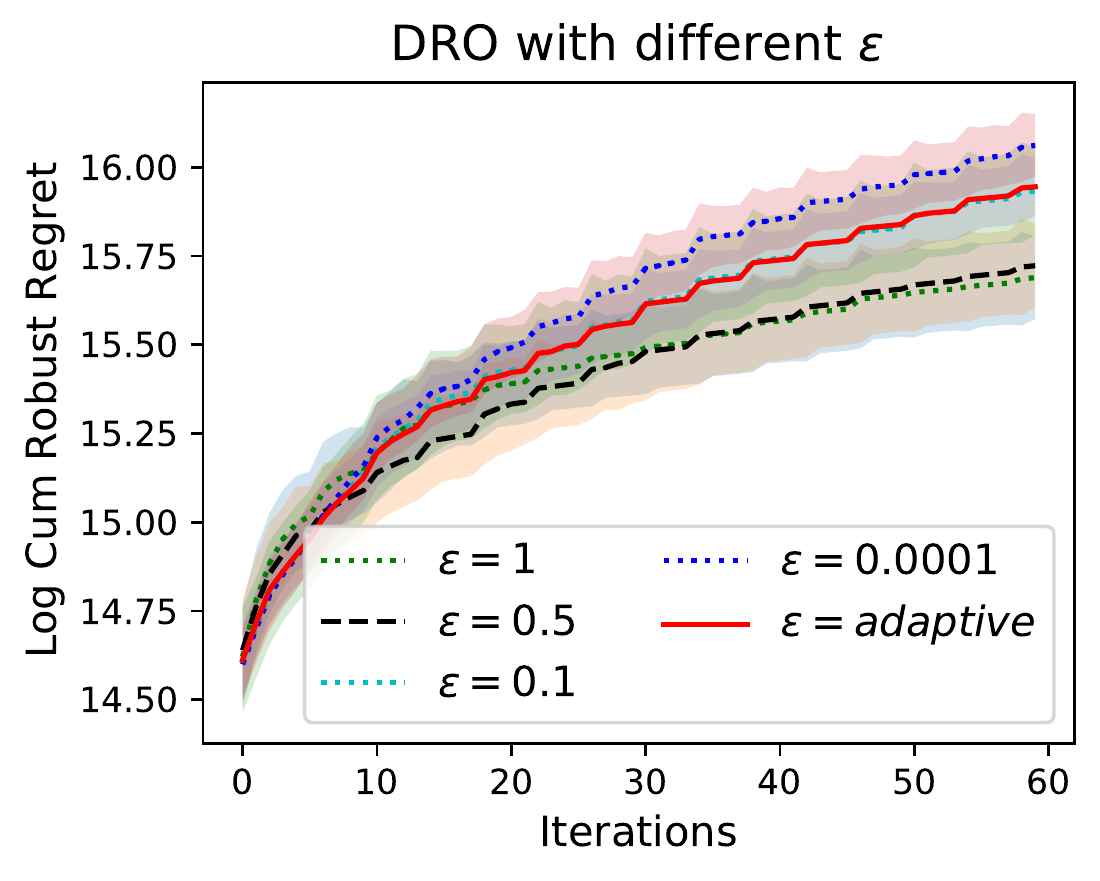}
%\vspace{-2pt}
\caption{Stochastic and DRO solution are different. The choices of $\epsilon=\{0.5,1\}$ result in the best performance.} \label{fig:appendix_sto_dro_different2}
\end{subfigure}

\begin{subfigure}[b]{1\textwidth}
\includegraphics[width=0.335\textwidth]{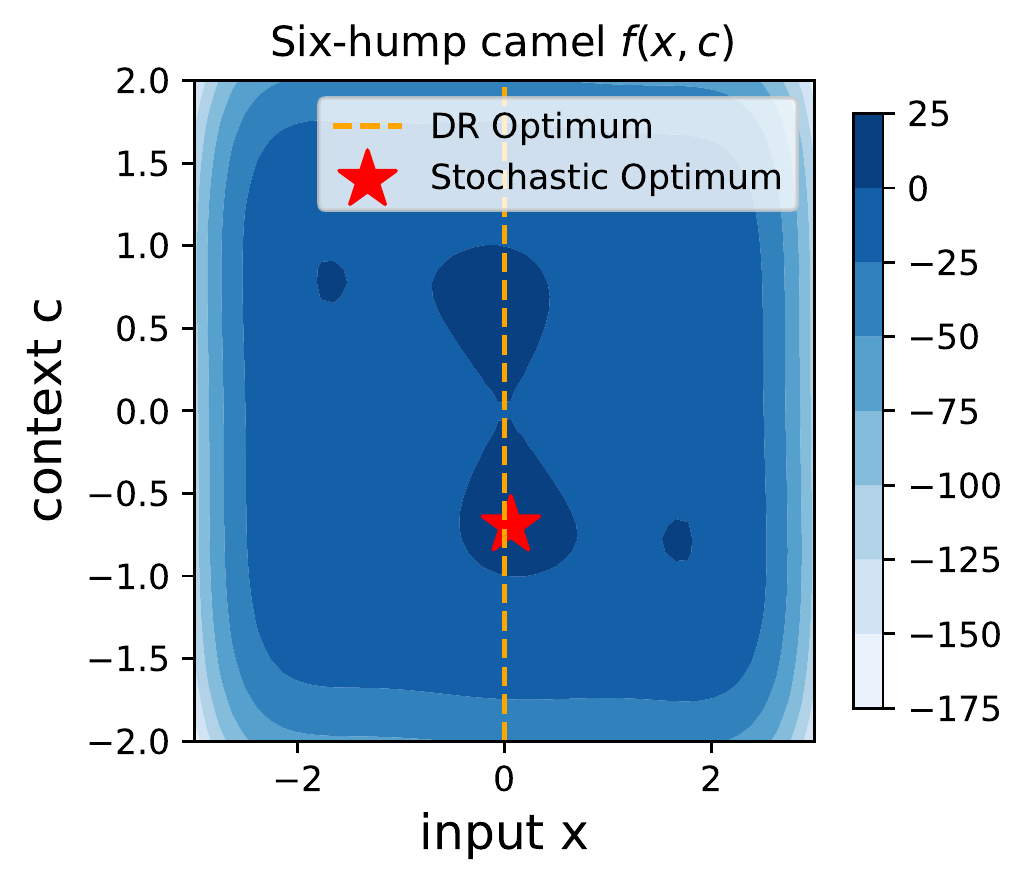}
\includegraphics[width=0.31\textwidth]{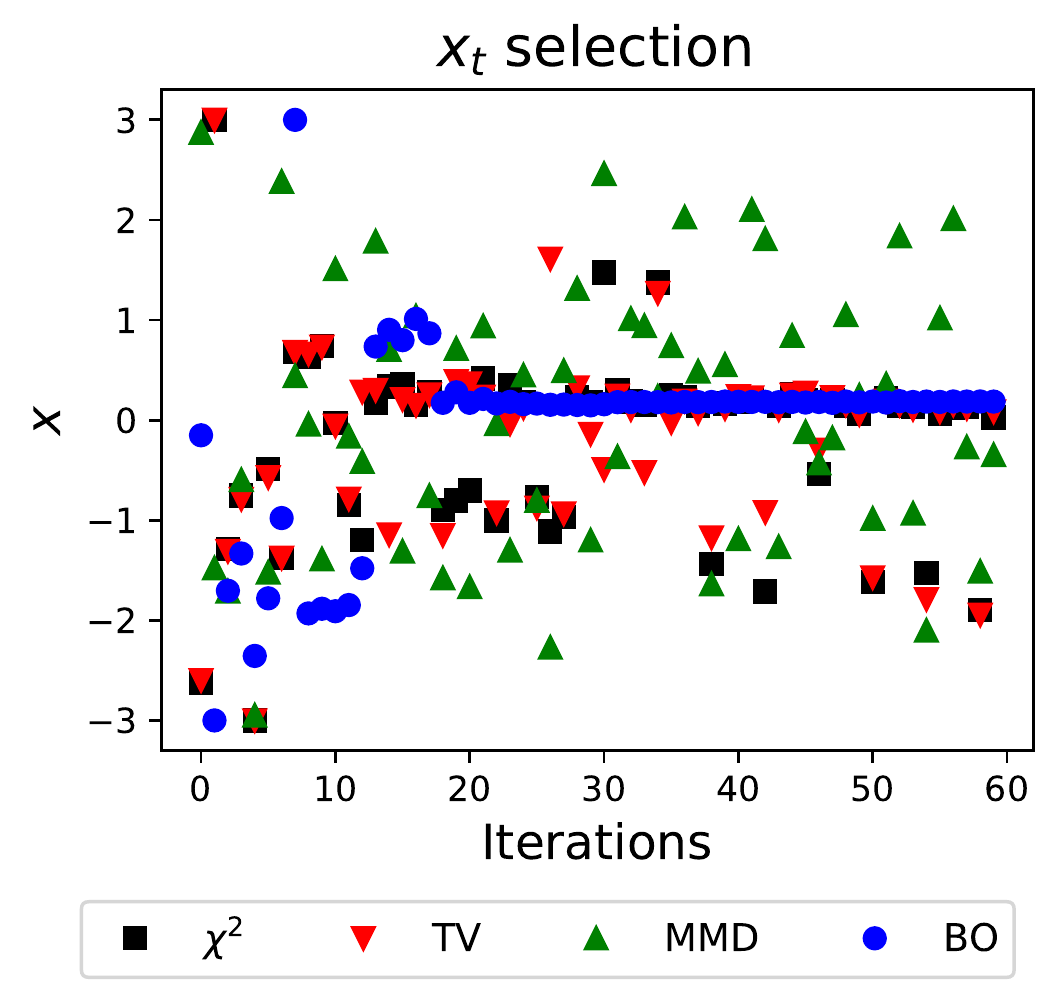}
\includegraphics[width=0.345\textwidth]{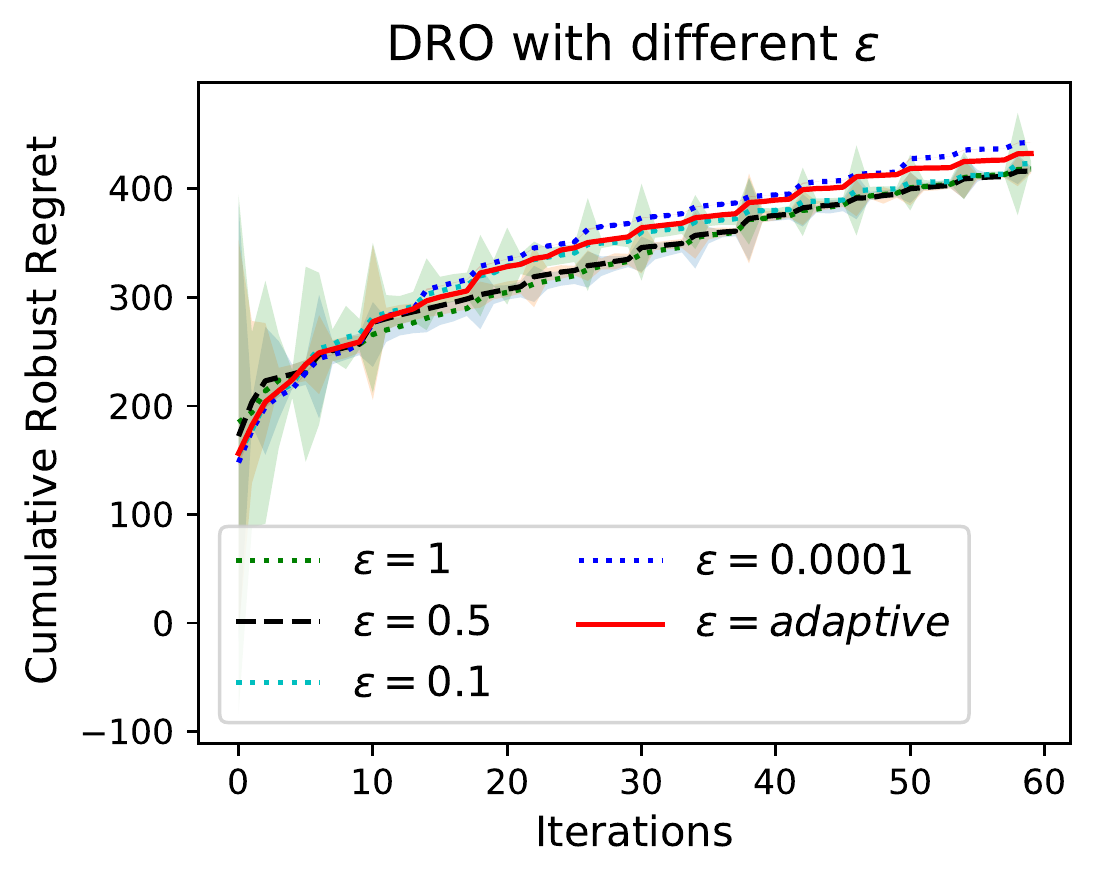}
%\vspace{-2pt}
\caption{Stochastic and DRO solution are coincide. $\epsilon \rightarrow 0$ is the best.} \label{fig:appendix_sto_dro_same}
\end{subfigure}

%\includegraphics[width=0.48\textwidth]{content/figs/branin_xselection.pdf}
%\includegraphics[width=0.48\textwidth]{content/figs/branin_yselection.pdf}
%\vspace{-8pt}
\caption{We complement the result presented in Fig. \ref{fig:ablation_study} using three additional functions. There are two settings in DRO when the stochastic solution and robust solution are different (\textit{top}) and identical (\textit{bottom}). \textit{Left}: the original function $f(\bx,c)$. \textit{Middle}: the selection of input  $\bx_t$  over iterations. \textit{Right}: optimization performance with different $\epsilon$. The adaptive choice of $\epsilon_t$ (in \textcolor{red}{red}) always produces stable performance across various choices of $\epsilon_t$. This is especially useful in unknown functions where we do not have prior assumption on the underlying structure to decide on which large or small values of $\epsilon_t$ to be specified.} \label{fig:ablation_study_appendix}
%\vspace{-8pt}
\end{figure*}

\subsection{Selection radii for $\varphi$-divergences}
Regarding the choice of radii $\varepsilon_t$, if we apply Gaussian smoothing on $p_t$ then there exists convergence rates for $\varphi$-divergences from \citet{goldfeld2020convergence}. Therefore we can select this choice and ensure that the population distribution $P$ is within the DRO-BO ball. We emphasize that we can choose smoothed distributions since the main Theorem holds for continuous context distributions, a feature specific to our contribution.

\end{document}